\PassOptionsToPackage{numbers,sort&compress}{natbib}
\documentclass[times, review, 10pt]{IEEEtran}  

\usepackage{amssymb}
\usepackage{amsmath}
\usepackage{amsthm}
\usepackage{graphicx}
\usepackage{mathtools}

\usepackage{lineno}
\usepackage[table,xcdraw]{xcolor}

\definecolor{lightgray}{gray}{0.9}

\usepackage{xcolor}
\usepackage{hyperref}
\usepackage{xspace}
\usepackage{comment}
\usepackage{multirow}
\usepackage{multicol}
\usepackage{arydshln}
\usepackage{float}
\usepackage{bbm}
\usepackage{tcolorbox}
\usepackage{booktabs}
\usepackage{xfrac}

\usepackage{caption}
\newcommand{\mytitle}{
On the Robustness of Adversarial Training Against Uncertainty Attacks
}

\newcommand*{\tinyimg}[1]{%
    \raisebox{-.1\baselineskip}{%
        \includegraphics[
        height=0.6\baselineskip,
        width=0.6\baselineskip,
        keepaspectratio,
        ]{#1}%
    }%
}

\newcommand{\tinypedicted}{\tinyimg{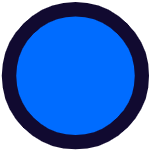}}
\newcommand{\tinygroundtruth}{\tinyimg{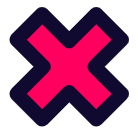}}
\newcommand{\tinyoverconfidence}{\tinyimg{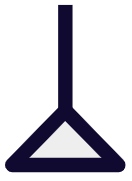}}
\newcommand{\tinyunderconfidence}{\tinyimg{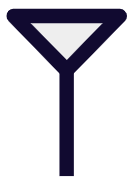}}

\newcommand{\myparagraph}[1]{\noindent \textbf{#1} --}

\newcommand{\SOA}{SoA\xspace}

\newcommand{\ie}{i.e.,\xspace}
\newcommand{\eg}{e.g.,\xspace}

\newcommand{\sece}{s-ECE\xspace}
\newcommand{\mus}{MUS\xspace}
\newcommand{\msus}{MSUS\xspace}

\newcommand{\cifargroupA}{$\diamondsuit$\xspace}
\newcommand{\cifargroupAvar}{$\diamondsuit^\prime$\xspace}
\newcommand{\cifargroupB}{$\clubsuit$\xspace}
\newcommand{\imagenetgroupA}{$\heartsuit$\xspace}
\newcommand{\imagenetgroupB}{$\spadesuit$\xspace}
\newcommand{\imagenetgroupBvar}{$\spadesuit^\prime$\xspace}

\newcommand{\param}{\ensuremath{\vct{\theta}}}
\newcommand{\robustparam}{\ensuremath{\vct{\theta}_{R}}}
\newcommand{\target}{\ensuremath{\vct t}}
\newcommand{\targetover}{\ensuremath{\vct t^{\downarrow}}}
\newcommand{\targetunder}{\ensuremath{\vct t^{\uparrow}}}
\newcommand{\deltaover}{\ensuremath{\vct \delta^{\downarrow}}}
\newcommand{\deltaunder}{\ensuremath{\vct \delta^{\uparrow}}}
\newcommand{\unc}[1]{\ensuremath{\mathcal{U}(#1)}}
\newcommand{\uncover}{\ensuremath{\mathcal{U}^{\downarrow}}}
\newcommand{\uncunder}{\ensuremath{\mathcal{U}^{\uparrow}}}

\newcommand{\vct}[1]{\ensuremath{\boldsymbol{#1}}}

\newcommand{\set}[1]{\ensuremath{\mathcal{#1}}}
\newcommand{\advset}[1]{\ensuremath{\mathcal{#1}_{\textrm{adv}}}}

\newcommand{\argmax}{\operatornamewithlimits{\arg\,\max}}

\newcommand{\argmin}{\operatornamewithlimits{\arg\,\min}}

\newcommand{\bucket}[1]{\mathbf{B}_{#1}}
\newcommand{\Loss}{\mathcal{L}}

\newcommand{\advLoss}{\mathcal{L}^{\prime}}

\newcommand{\trainset}{\ensuremath{\set D}\xspace}  

\newcommand{\x}{\vct{x}\xspace}

\newcommand{\deltaadv}{\vct{\delta}^{\prime}\xspace}

\newcommand{\partition}[1]{\mathbb{P}_{#1}}
\newcommand{\suchthat}{\; | \;}

\newcommand{\bigX}{\mathcal{X}}
\newcommand{\bigY}{\mathcal{Y}}
\newcommand{\datagendistrib}{\mathcal{P}(\bigX,\bigY)}
\newcommand{\classA}{\textrm{A}}
\newcommand{\classB}{\textrm{B}}
\newcommand{\epsball}{\set{B}_\epsilon}

\newtheorem{theorem}{Theorem}[section]

\newtheorem{lemma}{Lemma}[section]

\newcommand{\cifar}{CIFAR-10\xspace}
\newcommand{\cifarh}{CIFAR-100\xspace}
\newcommand{\mnist}{MNIST\xspace}
\newcommand{\imagenet}{ImageNet\xspace}
\newcommand{\atkit}{150\xspace}
\newcommand{\NumRobustBenchModels}{23\xspace}

\newcommand{\linf}{$l_\infty$\xspace}

\newcommand{\robustbench}{\cite{croce2021robustbench}\xspace}

\newcommand{\engstromname}{Engstrom2019Robustness\xspace}                               
\newcommand{\addepallieffname}{Addepalli2022Efficient\_RN18\xspace}                     
\newcommand{\gowalimprovingname}{Gowal2021Improving\_28\_10\_ddpm\_100m\xspace}         
\newcommand{\gowalimprovingduename}{Gowal2021Improving\_70\_16\_ddpm\_100m\xspace}      
\newcommand{\wangbettername}{Wang2023Better\_WRN-28-10\xspace}                          
\newcommand{\wangbetterduename}{Wang2023Better\_WRN-70-16\xspace}                       
\newcommand{\sehwagrobustname}{Sehwag2021Proxy\_R18\xspace}                             
\newcommand{\sehwagrobustduename}{Sehwag2021Proxy\_ResNest152\xspace}                   
\newcommand{\rebuffifixingname}{Rebuffi2021Fixing\_70\_16\_cutmix\_extra\xspace}        
\newcommand{\kangstablename}{Kang2021Stable\xspace}                                     
\newcommand{\pengrobustname}{Peng2023Robust\xspace}                                     
\newcommand{\addepallitowname}{Addepalli2021Towards\_RN18\xspace}                       
\newcommand{\cuidecoupledname}{Cui2023Decoupled\_WRN-28-10\xspace}                      
\newcommand{\xuexploringname}{Xu2023Exploring\_WRN-28-10\xspace}                        
\newcommand{\pangrobustnessname}{Pang2022Robustness\_WRN70\_16\xspace}                  

\newcommand{\engstromimagenetname}{Engstrom2019Robustness\xspace}           
\newcommand{\salmanname}{Salman2020Do\_R18\xspace}                          
\newcommand{\salmanduename}{Salman2020Do\_R50\xspace}                       
\newcommand{\wongfastname}{Wong2020Fast\xspace}                             
\newcommand{\liuswinbname}{Liu2023Comprehensive\_Swin-B\xspace}             
\newcommand{\liuswinlname}{Liu2023Comprehensive\_Swin-L\xspace}             
\newcommand{\liuconvnbname}{Liu2023Comprehensive\_ConvNeXt-B\xspace}        
\newcommand{\liuconvnlname}{Liu2023Comprehensive\_ConvNeXt-L\xspace}        

\newcommand{\engstromarch}{ResNet-50\xspace}
\newcommand{\addepallieffarch}{ResNet-18\xspace}
\newcommand{\gowalimprovingarch}{WideResNet-28-10\xspace}
\newcommand{\gowalimprovingduearch}{WideResNet-70-16\xspace}
\newcommand{\wangbetterarch}{WideResNet-28-10\xspace}
\newcommand{\wangbetterduearch}{WideResNet-70-16\xspace}
\newcommand{\sehwagrobustarch}{ResNet-18\xspace}
\newcommand{\sehwagrobustduearch}{ResNet-152\xspace}
\newcommand{\rebuffifixingarch}{WideResNet-70-16\xspace}
\newcommand{\kangstablearch}{WideResNet-70-16\xspace}
\newcommand{\pengrobustarch}{RaWideResNet-70-16\xspace}
\newcommand{\addepallitowarch}{ResNet-18\xspace}
\newcommand{\cuidecoupledarch}{WideResNet-28-10\xspace}
\newcommand{\xuexploringarch}{WideResNet-28-10\xspace}
\newcommand{\pangrobustnessarch}{WideResNet-70-16\xspace}

\newcommand{\engstromimagenetarch}{ResNet-50\xspace}
\newcommand{\salmanarch}{ResNet-18\xspace}
\newcommand{\salmanduearch}{ResNet-50\xspace}
\newcommand{\wongfastarch}{ResNet-50\xspace}
\newcommand{\liuswinbarch}{Swin-B\xspace}
\newcommand{\liuswinlarch}{Swin-L\xspace}
\newcommand{\liuconvnbarch}{ConvNeXt-B\xspace}
\newcommand{\liuconvnlarch}{ConvNeXt-L\xspace}

\newcommand{\engstrom}{\cite{robustness}\xspace}
\newcommand{\engstromID}{C1\xspace}

\newcommand{\addepallieff}{\cite{addepalli2022efficient}\xspace}
\newcommand{\addepallieffID}{C2\xspace}

\newcommand{\gowalimproving}{\cite{gowal2021improving}\xspace}
\newcommand{\gowalimprovingID}{C3\xspace}

\newcommand{\gowalimprovingIDdue}{C4\xspace}

\newcommand{\wangbetter}{\cite{wang2023better}\xspace}
\newcommand{\wangbetterID}{C5\xspace}

\newcommand{\wangbetterIDdue}{C6\xspace}

\newcommand{\sehwagrobust}{\cite{sehwag2021robust}\xspace}
\newcommand{\sehwagrobustID}{C7\xspace}

\newcommand{\sehwagrobustIDdue}{C8\xspace}

\newcommand{\rebuffifixing}{\cite{rebuffi2021fixing}\xspace}
\newcommand{\rebuffifixingID}{C9\xspace}

\newcommand{\kangstable}{\cite{kang2021stable}\xspace}
\newcommand{\kangstableID}{C10\xspace}

\newcommand{\pengrobust}{\cite{peng2023robust}\xspace}
\newcommand{\pengrobustID}{C11\xspace}

\newcommand{\addepallitow}{\cite{addepalli2022scaling}\xspace}
\newcommand{\addepallitowID}{C12\xspace}

\newcommand{\cuidecoupled}{\cite{cui2023decoupled}\xspace}
\newcommand{\cuidecoupledID}{C13\xspace}

\newcommand{\xuexploring}{\cite{xu2023exploring}\xspace}
\newcommand{\xuexploringID}{C14\xspace}

\newcommand{\pangrobustness}{\cite{pang2022robustness}\xspace}
\newcommand{\pangrobustnessID}{C15\xspace}

\newcommand{\engstromimagenetID}{I1\xspace}
\newcommand{\salman}{\cite{salman2020adversarially}\xspace}
\newcommand{\salmanID}{I2\xspace}
\newcommand{\salmanIDdue}{I3\xspace}
\newcommand{\wongfast}{\cite{wong2020fast}\xspace}
\newcommand{\wongfastID}{I4\xspace}
\newcommand{\liu}{\cite{liu2023comprehensive}\xspace}
\newcommand{\liuswinb}{I5\xspace}
\newcommand{\liuswinl}{I6\xspace}
\newcommand{\liuconvnb}{I7\xspace}
\newcommand{\liuconvnl}{I8\xspace}

\begin{document}

\title{\mytitle\\
{\footnotesize Preprint, submitted to Pattern Recognition}
\thanks{Submitted to Pattern Recognition.\\ * Corresponding author}
}

\author{\IEEEauthorblockN{Emanuele Ledda$^{1,*}$}
\and
\IEEEauthorblockN{Giovanni Scodeller$^3$}
\IEEEauthorblockA{}
\and
\IEEEauthorblockN{Daniele Angioni$^2$}
\IEEEauthorblockA{}
\and
\IEEEauthorblockN{Giorgio Piras$^2$}
\IEEEauthorblockA{}
\and
\IEEEauthorblockN{Antonio Emanuele Cin\`a$^3$}
\IEEEauthorblockA{}
\and
\IEEEauthorblockN{Giorgio Fumera$^2$}
\IEEEauthorblockA{}
\and
\IEEEauthorblockN{Battista Biggio$^{1,2}$}
\IEEEauthorblockA{}
\and
\IEEEauthorblockN{Fabio Roli$^{1,2,3}$}
\IEEEauthorblockA{}

\IEEEauthorblockA{\textit{(1) CINI, Consorzio Interuniversitario Nazionale per per l'Informatica} \\
Rome 00185, Italy}

\IEEEauthorblockA{\textit{(2) Department of Electric and Electronic Engineering, University of Cagliari} \\
Cagliari 09100, Italy}

\IEEEauthorblockA{\textit{(3) Department of Informatics, Bioengineering, Robotics, and Systems Engineering, University of Genova} \\
Genova 16146, Italy}
}

\maketitle

\begin{abstract}
In learning problems, the noise inherent to the task at hand hinders the possibility to infer without a certain degree of uncertainty. 
Quantifying this uncertainty, regardless of its wide use, assumes high relevance for security-sensitive applications.
Within these scenarios, it becomes fundamental to guarantee good (\ie trustworthy) uncertainty measures, which downstream modules can securely employ to drive the final decision-making process.
However, an attacker may be interested in forcing the system to produce either (i) highly uncertain outputs jeopardizing the system's availability or (ii) low uncertainty estimates, making the system accept uncertain samples that would instead require a careful inspection (\eg human intervention). 
Therefore, it becomes fundamental to understand how to obtain robust uncertainty estimates against these kinds of attacks.
In this work, we reveal both empirically and theoretically that defending against adversarial examples, \ie carefully perturbed samples that cause misclassification, additionally guarantees a more secure, trustworthy uncertainty estimate under common attack scenarios without the need for an ad-hoc defense strategy.
To support our claims, we evaluate multiple adversarial-robust models from the publicly available benchmark RobustBench on the CIFAR-10 and ImageNet datasets.
\end{abstract}






\section{Introduction}
In recent years, the presence of Machine Learning (ML) has become increasingly widespread, mainly owing to its powerful predictive abilities. 
However, the use of ML models, such as any system learning from data, is tied to uncertainty. 
The presence of variability in the data observations, the limit on the number of observations, and, more generally, the presence of noise on the task at hand make models uncertain~\cite{hullermeier2021aleatoric}. 
Thus, it has become fundamental to represent uncertainty in a reliable and informative way through the use of Uncertainty Quantification (UQ) methods, enabling an aware decision-making process~\cite{McAllister2017Autonomous}.

The role of UQ techniques assumes a highly pivotal role in security-sensitive applications: in such scenarios, in fact (\eg medical diagnostics~\cite{guo2024uctnet, wei2025fingrained} and autonomous driving~\cite{McAllister2017Autonomous}), the tolerance on the error margin must be, by definition, minimized and strictly confined.
Leveraging UQ techniques, however, ML-based systems can also base their output on the uncertainty estimated on the given input, helping minimize errors and enabling an accurate and reliable decision-making process throughout the ML systems' pipelines.

Nevertheless, according to a recent promising line of research, the uncertainty estimates can also be subject to attacks aiming to maliciously tamper with the uncertainty~\cite{galil2021disrupting, zeng2022outdomain, Ledda_2023_ICCV, obadinma2024calibration}. 
Specifically, an attacker can be interested in jeopardizing the availability of the ML-based system, typically by increasing the uncertainty of input samples to ``flood" the decision-making system. 
Conversely, by targeting the system's integrity, an attacker can be interested in decreasing the uncertainty to make the system accept inputs that would otherwise need to undergo the decision-making system, \eg a human in the loop.  

Within the same security-related contexts, however, offspring of an instead now mature line of research, ML models are also often required to be resistant against adversarial inputs, \ie carefully crafted input samples aiming to deceive the classification performed by the model~\cite{szegedy2014intriguing, biggio2013evasion}. 
Following an arms race in the security of ML, the literature has flourished, proposing multiple defense approaches and attack algorithms~\cite{biggio2018wild}. 
Among the State-Of-the-Art (\SOA) defense approaches, Adversarial Training (AT)~\cite{madry2017towards} has been shown to be the reference technique by also being the pillar of more recent approaches, and its effectiveness is deeply and continuously studied against the \SOA adversarial attacks~\cite{croce2021robustbench}.

In contrast, for attacks against uncertainty, little to no work has advanced the study of defensive approaches, rather than focusing on different ways to disrupt the UQ techniques. 
We thus questioned, starting from the principle of evaluating existing approaches first, the effectiveness of AT against attacks on uncertainty.
In our work, we reveal, through theoretical analysis and empirical investigation, that models trained to be robust against adversarial attacks can be likewise robust against attacks to uncertainty; we summarize the main theoretical findings of our analysis in ~\autoref{fig:graphical_abstract}.
\begin{figure}[tb]
    \centering
    \includegraphics[width=\linewidth]{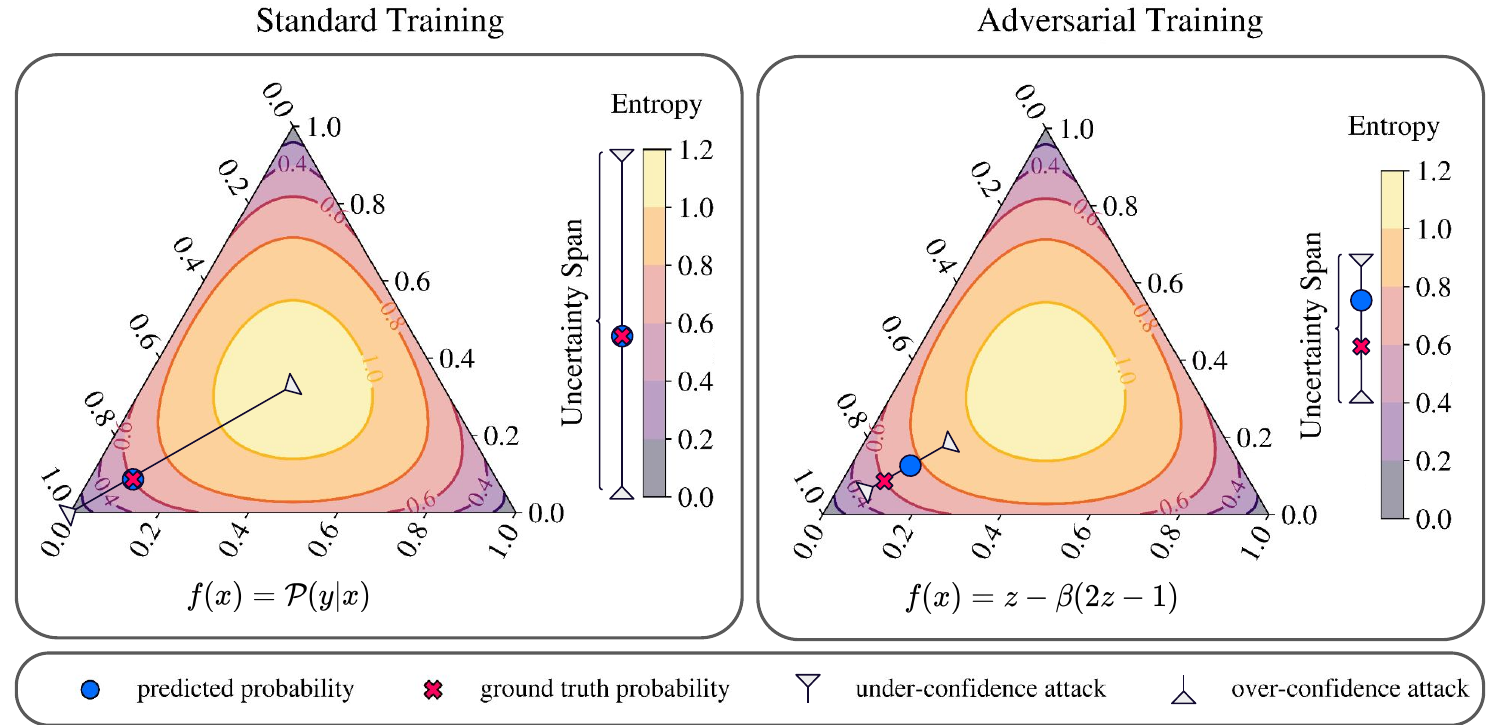}
    \caption{An illustration of the effect of standard (left) and adversarial (right) training on the Uncertainty Span (\ie the range of uncertainty values one can obtain by perturbing a sample with a given budget $\epsilon$) for a classification problem visualized through a standard 3-dimensional simplex.
    While for standard training, the predicted probability \tinypedicted~converges on the ground truth probability \tinygroundtruth, for adversarial training, it results in less confidence, \ie has higher entropy than the ground truth probability.
    As a result, adversarial training reduces the uncertainty span, whose lower and upper bounds are, respectively, an over-confidence \tinyoverconfidence~and under-confidence \tinyunderconfidence~attack.
    For further details, see ~\autoref{sec:method} and ~\autoref{sec:theory}.}
    \label{fig:graphical_abstract}
\end{figure}

In \autoref{sec:bg}, we display the background concepts of uncertainty quantification and adversarial ML; then, in \autoref{sec:method}, starting from the state of the art, we exemplify a modular framework for dealing with adversarial attacks targeting uncertainty, which is functional to the theoretical analysis on the robustness of adversarial trained models of \autoref{sec:theory}.
Successively, in \autoref{sec:experiments}, we empirically demonstrate our results by considering, on multiple datasets, highly robust \SOA models from the publicly available RobustBench repository.
Finally, in \autoref{sec:conclusions}, we outline our findings and discuss the potential limitations of the study at hand.
%
%
\section{Background}
\label{sec:bg}

\subsection{Uncertainty Quantification}
\label{sec:bg.uq}
Uncertainty Quantification (UQ) is an emerging ML field devoted to measuring uncertainty (usually, the uncertainty associated with models' predictions)~\cite{hullermeier2021aleatoric}.
This uncertainty, commonly known as \emph{predictive uncertainty}, can be decomposed into two distinct sources~\cite{hullermeier2021aleatoric}: (i) \textbf{aleatoric uncertainty}: which refers to the inherent randomness of the decision process; (ii) \textbf{epistemic uncertainty}: caused by a lack of knowledge.
An example of aleatoric uncertainty is the ambiguity when trying to discriminate the hand-written digit ``5'' from the letter ``S''.
In some extreme cases, the similarities between the strokes composing these two symbols can make them even indistinguishable.
Unlike the previous example, epistemic uncertainty has no relation to any inherent randomness of the process: for example, imagine determining if a mushroom is poisonous by its appearance.
Without proper knowledge of the target domain, one would be uncertain in determining if a given mushroom would be poisonous; however, knowledge acquisition reduces the uncertainty of the predictions (the reason for which we also refer to this source as ``reducible uncertainty'').

\myparagraph{Bayesian approach to uncertainty quantification}
One solid approach for quantifying these two sources of uncertainty consists of estimating a posterior distribution on the set of the model's parameters $\param$ given the data $\set{D}$:
\begin{equation}
    p(\param|\set{D}) = \frac{p(\set{D|\param}) \cdot p(\param)}{p(\set{D})} .
\end{equation}
From this distribution, one can obtain a prediction by Bayesian Model Averaging (BMA), \ie by computing the expected value of the prediction's posterior $p(y|\vct{x},\param)$ given the parameters $\param$ with respect to the parameter's posterior $p(\param|\set{D})$:
\begin{equation}
    f^{BMA}(\x) = p(y|\x,\set{D}) = \int_{\Theta} p(y|\x,\param) \cdot p(\param|\set{D}) \; \partial \theta ,
\label{eq:bma}
\end{equation}
where $\Theta$ is the set of all the possible parameters.
For obtaining uncertainties associated directly with the model's prediction, the standard approach consists of using a second-order probability over the prediction's Bayesian posterior: by convention, the first-order moment of this distribution is ``believed" to embed information about the aleatoric uncertainty, while the second order moment is \emph{directly} associated to epistemic uncertainty.

A standard approach consists of using the entropy of the predictive vector as a measure of aleatoric uncertainty.
Practically, for classification problems, such a solution acts directly on the predicted softmax vector encoding the prediction $f(\x)$:
\begin{equation}
    H(f(\x)) \coloneqq -\sum_{c=1}^{|\set{Y}|} f_{c}(\x) \log{f_{c}(\x)},
\end{equation}
where $f_c(\x)$ denotes the probability value the model assigns to the class $c$, with $f(\x)$ which can be obtained either by using a Bayesian model $f(\x) \coloneqq f^{BMA}(\x)$ or a deterministic one $f(\x) \coloneqq f^{\theta}(\x)$ parametrized with $\theta$.
Correctly modeled entropies should convey calibration properties, \ie when the model assigns a probability value of $p$ to an event (\eg belonging to a specific class), it actually occurs with a probability of $p$.
The standard metric for quantifying this behavior takes the name of Expected Calibration Error (ECE)~\cite{guo2017calibration}, which is a discrete expectation of the absolute difference between the expected ${ex}(\bucket{s})$ and the observed ${ob}(\bucket{s})$ probability computed on a set of $S$ buckets $\bucket{s}$:
\begin{equation}
\text{ECE} = \sum^S_{s=1} \frac{|\bucket{s}|}{|\set{D}|} |{ob}(\bucket{s})-{ex}(\bucket{s})|
\label{eq:ece}
\end{equation}
where each $\bucket{s} \in \set{D}$ contains all the samples having confidence (\ie $\arg \max f(\x)$) within the interval $[\frac{s}{S}, \frac{s+1}{S}]$, and ${ex}(\bucket{s})=\frac{2s+1}{2S}$ is the average over the interval.
For measuring epistemic uncertainty, a common approach employs the variance of the predictions: $\mathbb{V}(f^{BMA}(\x)) \coloneqq \mathbb{E}[(f^{BMA}(\x))^2] - \mathbb{E}[(f^{BMA}(\x))]^2$.

\subsection{Adversarial Attacks}
\label{sec:bg.adv}
In a security-sensitive application, another priority, together with reliable uncertainty estimates, is to guarantee a certain level of robustness in a model that can be the target of attackers aiming at disrupting its normal functioning.
This necessity brought to light what is known as \textit{adversarial machine learning}, an ML field that keeps exploring the attack surfaces of an ML system to propose new countermeasures to achieve robustness~\cite{biggio2018wild}.
In this regard, most of the research community efforts focused on attacks jeopardizing the integrity of a model, where an attacker is interested in changing a model's prediction at test time~\cite{biggio2013evasion, szegedy2014intriguing}.
In such an attack scenario, the goal of the attacker is to craft a small perturbation $\deltaadv$ such that, once added to an input sample $\x$, it induces the target model $f^{\param}$ to output a wrong prediction, \ie $f^{\param}(\x + \deltaadv) \neq f^{\param}(\x)$.
Here, the perturbation is typically constrained to a limited domain $\epsball = \{\forall \vct \delta\ : ||\vct \delta||_p \leq \epsilon\}$, where $||.||_p$ is an arbitrary $L_p$ norm (\eg $L_1$, $L_2$ or $L_\infty$) and $\epsilon$ is the perturbation budget indicating the limits of the attacker's power.

\myparagraph{Optimizing the adversarial perturbation}
To find the optimal perturbation $\deltaadv$ that maximizes the probability of success for the attacker, one can solve the following optimization problem:
\begin{equation}
\label{eq:adv_objective}
    \deltaadv = \argmin_{\vct \delta \in \epsball} \advLoss(\vct t, f^{\param}(\x + \vct \delta))
\end{equation}

where $\vct t$ is the one-hot encoding of the target label, which can be either any label different than the ground truth or a specific target class chosen by the attacker, and $\advLoss$ is the attacker's objective, \eg the negative cross-entropy between the prediction and the target or the logit $f_k(\x + \deltaadv)$ corresponding to the correct class.

\myparagraph{Defending against adversarial attacks}
Currently, the most promising approach to protect a model against adversarial attacks is to proactively anticipate the attacker by including them during training. Such a procedure, which takes the name of \textit{adversarial training} (AT)~\cite{madry2017towards}, aims at finding the optimal parametrization $\param^\prime$ that ideally protects from all possible perturbations $\vct \delta \in \epsball$.
This strategy can be formalized as the following minimax problem:
\begin{equation}
\label{eq:advtraining}
\min_{\param \in \Theta}
\mathbb{E}_{\vct x, y \sim \trainset}
\left[
\max_{\vct \delta \in \epsball} \Loss(f^{\param}(\x + \vct \delta), y)
\right]
\end{equation}
where $\set{L}$ is the objective that minimize classification error (\eg the cross-entropy loss).
Here, \autoref{eq:advtraining} is composed of (i) an \textit{inner maximization} problem to find, for each sample, the perturbation $\vct \delta$ that achieves high loss, and (ii) an \textit{outer minimization} problem to find the parameters $\param$ that minimize this worst-case loss.
%
%
\section{Adversarial Attacks Against Uncertainty Quantification}  
\label{sec:method}
From the perspective of the most typical test-time evasion attacks, adversarial machine learning characterizes an attacker interested in making the model misclassify the prediction on the input adversarial example.
In the context of Uncertainty Quantification, instead, the focus shifts from ``deteriorating predictions'' to ``deteriorating the uncertainty measure''.
The correct functioning of an uncertainty measure implies that the most likely correct predictions have a low uncertainty, which, in contrast, becomes high when the predictions are plausibly incorrect. 
Therefore, while predictions and uncertainty represent two different measures, they are, however, affine: the uncertainty estimate measures the likelihood that the given prediction is wrong. 
As we will see in the following subsection, state-of-the-art work conceived uncertainty adversarial attacks specifically to spoil this statistical correlation.


\subsection{Previous Work}
\label{sec:method.approaches}
Galil and El-Yaniv~\cite{galil2021disrupting} are the first to consider attacks to disrupt the uncertainty measure.
Specifically, they consider using such a measure for building a classifier with a rejection option.
They aim to subvert the rejection decision (\ie to increase the probability of rejecting correctly classified samples and decrease the probability of rejecting the misclassified ones) without changing the predicted labels so that a user would not notice the attack.
To do so, they selectively increase or decrease the confidence assigned to each sample \textit{depending on the model's prediction before the attack} by using the Maximum Softmax Probability (MSP) as loss function $\ell(\x) = f_{\hat{y}}(\x)$, where $\hat{y}$ is the predicted class:
\begin{equation}
\begin{aligned}
\argmin_{\vct \delta \in \epsball} \quad & \gamma \cdot f_{\hat{y}}(\x+\vct \delta) \\
\textrm{s.t.} \quad & f_{\hat{y}}(\x+\vct \delta) \ge f_{c}(\x+\vct \delta), \quad \forall c \in \set{Y}
\end{aligned}
\label{eq:galil_opt}
\end{equation}
where $\gamma=-1$ is used for misclassified samples and $\gamma=1$ for correctly the classified ones.
Zeng et al.~\cite{zeng2022outdomain} model an attacker interested in making Out-of-Distribution (OOD) samples appear as in-distribution samples.
To do so, they propose to minimize the cross-entropy between the adversarial prediction $f(\x_{out} + \vct \delta)$ for an OOD sample $\x_{out}$ and the target $\target$ representing the one-hot encoding of the predicted label $\hat{y}$:
\begin{equation}
\begin{aligned}
\argmin_{\vct \delta \in \epsball} \quad & H(f(\x_{out} + \vct \delta), \target) 
\end{aligned}
\label{eq:zeng_opt}
\end{equation}
It is worth noting that both the techniques mentioned above~\cite{galil2021disrupting,zeng2022outdomain} design an attacker who has access to prior knowledge about input annotations.
In the former, the ground truth is crucial to determine when to increase or decrease the confidence, while, in the latter, the attacker needs to know which samples are OOD (\ie implicitly not belonging to any known class) to perform the attack.

In our previous work~\cite{Ledda_2023_ICCV}, we proposed an alternative scenario in which the attacker is interested in increasing or decreasing the predictive uncertainty of all the input samples.
The ability to damage predictive uncertainty without having specific information about the labels is an attractive feature, especially if such a measure is the input of a downstream module or is given to a human.
For instance, imagine a system employing the uncertainty measure of an ML model for determining if the clinical data of a patient is sufficient to make a diagnosis: without knowledge of this specific domain (thus without the labels), an attacker may craft a perturbation for increase indiscriminately the uncertainty associated to the patient's data.
Such an attack may overload the processing pipeline of a hospital, causing the system to require more analysis, even when unnecessary.
%
The proposed attack scenario takes the name of \textit{over-confidence} (when decreasing the uncertainty) or \textit{under-confidence} attack (when increasing the uncertainty).
While for the over-confidence attack, we kept the same optimization process of ~\autoref{eq:zeng_opt}, the under-confidence case has been investigated only theoretically, without an actual loss function implementation.

Obadinma et al.\cite{obadinma2024calibration} propose to harm the model's calibration, disrupting the functioning of any downstream module or human using the uncertainty measure.
For doing so, they present four different attacks, two of which coincide with the formulation we proposed in ~\cite{Ledda_2023_ICCV}, \ie (i) over-confidence and (ii) under-confidence attack.
They further propose (iii) \textit{maximum miscalibration attack} (MMA), where they increase or decrease the confidence of each sample following the same criteria of ~\cite{galil2021disrupting}, and (iv)
\textit{random confidence attack} (RCA), where they extend the latter but with a randomized confidence goal for each input, producing a more natural confidence distribution that, although being a less effective attack, results in being more stealthy.
Based on these types of attacks against model's calibration, they propose two ad-hoc defenses: (i) \textit{Calibration Attack Adversarial Training} (\textit{CAAT)}, in which they perform adversarial training using samples generated with maximum miscalibration attacks, and (ii) \textit{Compression Scaling} (\textit{CS}), a post-processing technique where they heuristically distribute low scores to higher range.
To the best of our knowledge, this is the only work that proposes a defense strategy against uncertainty attacks.

To the best of our knowledge, all state-of-the-art approaches rely exclusively on aleatoric uncertainty estimates derived from deterministic models, whereas epistemic uncertainty in the context of adversarial robustness has been explored solely in our prior work~\cite{Ledda_2023_ICCV}.
This common practice is not a deliberate choice, but rather a consequence of the unpopularity of Bayesian models among the adversarial training community, which makes robust models capable of providing principled epistemic measures hard to find.

\subsection{A Modular Unified Framework for Uncertainty Attacks}
\label{sec:method.framework}
Interestingly, all the state-of-the-art works build an optimization process with a proper loss formulation for accomplishing an attacker's goal.
This means that proving the model's robustness against a specific goal does not directly translate to the same robustness guarantees on other goals.
For example, analyzing the effect of an attack on OOD uncertainty minimization does not provide any information on the effectiveness against indiscriminate uncertainty maximization, making it difficult to compare models in terms of uncertainty robustness.
Accordingly, we propose a modular approach for unifying the diverse attacks proposed in the literature, which, as a result, also unifies the notion of robustness.
In general, all the possible instances of the attacker's goal have a shared characteristic: they all need to modify the uncertainty measure associated with the model's predictions; more sensitive uncertainty variations generally lead to less robustness against uncertainty attacks.
%
As a result, any possible uncertainty attack discussed in literature (\eg Maximum Miscalibration~\cite{obadinma2024calibration}, Random Confidence~\cite{obadinma2024calibration}, and OOD camoufladge~\cite{galil2021disrupting}) can be traced back to a combination of over- and under-confidence attacks.
Let us take random confidence as an example: knowing which is the maximum and minimum attainable uncertainty value under attack for a specific set of samples directly influences the success rate of the random confidence attack (the higher the possible excursion, the higher the space of action of the attack).
As another example, given that for succeeding in a maximum miscalibration attack it is necessary to selectively increase or decrease the uncertainty of different samples, the attack success rate depends upon the robustness of the models against under- and over-confidence attacks.

Therefore, instead of focusing on specific instances of the attacker’s goal, we \textit{estimate the range of uncertainty values an attacker can reach} when perturbing each input $x$ on the dataset $\set{D}$.
This approach provides a comprehensive overview of the model’s vulnerability under uncertainty manipulations, encompassing all possible attack scenarios.
%
Concretely, given an attacker capable of perturbing a sample $\x$ inside a domain $\epsball$, we can distinguish between two different cases: (i) decreasing the uncertainty (\ie over-confidence attack) and (ii) increasing the uncertainty (\ie under-confidence attack).
These objectives can be achieved, respectively, by minimising or maximising the uncertainty measure $\unc{f^{\param}(\x)}$:
\begin{equation}
\label{eq:unc_generic_obj}
     \underbrace{\deltaover =\argmax_{\vct \delta \in \epsball} \; \unc{f^{\param}(\x + \vct \delta)}}_{\textbf{Over-confidence Attack}} 
    \quad 
    \underbrace{\deltaunder =\argmin_{\vct \delta \in \epsball} \; \unc{f^{\param}(\x + \vct \delta)}}_{\textbf{Under-confidence Attack}}
\end{equation}
Here, $\deltaover$ represents the perturbation for which the attacker maximizes the uncertainty (under-confidence attack) while $\deltaunder$ represents the perturbation for which it is minimized (over-confidence attack).
Starting from them, we can compute the lower $\uncover \coloneq \unc{f^{\param}(\x + \deltaover)}$ and upper $\uncunder \coloneq \unc{f^{\param}(\x + \deltaunder)}$ bound of what we call the \textbf{Uncertainty Span} (US), \ie all the uncertainty values an attacker can obtain when perturbing the sample $\x$ for fooling a model parameterized by $\param$ given a perturbation budget $\epsilon$.
With such an approach, one can easily derive the robustness to specific uncertainty attacks, \eg by selectively choosing the upper/lower bound of the uncertainty range~\cite{galil2021disrupting, zeng2022outdomain} or even intermediate values~\cite{obadinma2024calibration}.

In accordance with the state of the art, we use the entropy of the prediction vector $H(f^{\param}(\x))$~\cite{hullermeier2021aleatoric} as our uncertainty measure $\unc{\cdot}$.
To instantiate an over-confidence attack, the objective is to minimize the entropy of the output probability.
To do so, it is sufficient to force the model to output $f^{\param}_k(\x + \deltaover) = 1$ for an arbitrary $k \in \set Y$.
%
%
We know from \cite{Ledda_2023_ICCV} that the most straightforward strategy for the attacker is to maximize the probability of the class corresponding to the predicted class, as it already has the highest probability among all the classes, as in ~\autoref{eq:zeng_opt}; this attack takes the name of Stabilizing Attack (STAB) since it stabilizes the model's prediction by enforcing the most likely class~\cite{Ledda_2023_ICCV}.
To achieve this objective, one can employ the cross-entropy loss, setting as target $\targetover$ the one-hot-encoding of the predicted label $\hat{y}$:
\begin{equation}
\label{eq:oatk_objective}
    \argmin_{\deltaover \in \epsball} H(\targetover, f^{\param}(\x+\deltaover)) = 
    \argmin_{\deltaover \in \epsball} - \log(f^{\param}_{\hat{y}}(\x + \deltaover)),
\end{equation}
which finds its minimum when $f_{\hat{y}}(\x + \deltaover) = 1$.
In the case of under-confidence attacks, we carry out a similar approach by using the cross-entropy: 
to this end - since the entropy reaches its maximum when the output is a uniform vector - we set the target $\targetunder = \frac{1}{c} \cdot \mathbbm{1}$ to be a $c$-dimensional vector with all elements equal to $\frac{1}{c}$, where $c$ is the total number of classes:
\begin{equation}
\label{eq:uatk_objective}
    \argmin_{\deltaunder \in \epsball} H(\targetunder, f^{\param}(\x+\deltaunder)) = 
    \argmin_{\deltaunder \in \epsball} -\frac{1}{c}\sum_{k=1}^c \log(f^{\param}_k(\x + \deltaunder)).\\
\end{equation}
%
Finally, for a thorough evaluation, we must extend the notion of the Uncertainty Span (a sample-wise measure) onto an entire data set.
To this aim, we propose to compute general statistics by taking the mean of the uncertainty spans\footnote{By abuse of notation, we refer to Uncertainty Span for indicating both the interval and the interval size.} (MUS):
\begin{equation}
\label{eq:mus}
    \text{MUS} = 
    \mathbb{E}_{(\x,y) \sim \trainset}
    \Biggl[
            H(f^{\param}(\x + \deltaunder)) - H(f^{\param}(\x + \deltaover))
    \Biggr] \,
\end{equation}
We can also penalize larger gaps by employing a \emph{squared} mean (MSUS):
\begin{equation}
\label{eq:msus}
    \text{MSUS} = 
    \mathbb{E}_{(\x,y) \sim \trainset}
    \Biggl[
            \Bigl( 
                H(f^{\param}(\x + \deltaunder)) - H(f^{\param}(\x + \deltaover))
            \Bigr) ^{2}
    \Biggr] \,
\end{equation}
%
%
\section{Theoretical Analysis on Adversarial Uncertainty Robustness}  
\label{sec:theory}

Despite much literature on the attacker side, only one work dealt with adversarial robustness against uncertainty attacks, proposing two ad hoc strategies.
In contrast to them, driven by the flourishing literature on defenses against prediction attacks (\ie traditional adversarial examples), before thinking of ad hoc strategies, we ask ourselves if and to what extent such defenses, particularly adversarial training, may protect even against uncertainty attacks.
Accordingly, our main research question is:
\begin{tcolorbox}
Do \textbf{adversarial trained} models disclose robust characteristics also \textbf{against attacks targeting uncertainty}? 
\end{tcolorbox}
A recent work~\cite{obadinma2024calibration} likewise raised this observation, but it is still limited to empirically analyzing a small set of models and only against specific attack instances.
To this aim, instead of considering specific attack instances, we answer our research question by analyzing the Uncertainty Span of adversarial trained models, \ie consider both the under- and over-confidence cases, in accordance with ~\autoref{sec:method}, drawing \emph{general} conclusions on these models' robustness against \emph{any} attack instance.
Since the predictive entropy of a sample increases when approaching the decision boundary, it is reasonable to think that adversarial training may also protect against under-confidence attacks: in fact, adversarial training forces the model to be robust to label-flips (\ie to samples crossing the boundary), hindering any potential attacker interested in increasing uncertainty.
For over-confidence attacks instead, we cannot draw similar conclusions: indeed, being robust to label flips does not guarantee an entropy reduction since the regions where the entropy is minimized do not lie on the decision boundary.

Until now, we have provided a qualitative and intuitive discussion; now, let us move on to a formal, theoretical analysis.

\subsection{On the Convergence of Adversarial Trained Models}
\label{sec:method.proof}
For simplicity, we consider a binary classification problem:
let $\datagendistrib$ be a generative distribution; sampling from this distribution results in obtaining a pair $(\x,y)$ characterized by an observation $\x \in \mathbb{R}^d$ and a target binary label $y \in \{A,B\}$.
Let then $\set{D}$ be a collection of i.i.d. pairs $(\x,y)$ obtained sampling the generative distribution $\datagendistrib$:
\begin{equation}
    \set{D} = \{(\x_i, y_i) \sim \datagendistrib, \; i = [1, \cdots, N] \} ,
\end{equation}
Note that for the law of large numbers, when the cardinality of the data set $|\set{D}|$ increases, the probability $p(y|\x,\set{D})$ conditioned to the data set tends to $p(y|\x)$:
\begin{equation}
\label{eq:largenumbers}
    \lim_{|\set{D}| \rightarrow \infty} p(y|\x,\set{D}) =
    p(y|\x) .
\end{equation}
\autoref{eq:largenumbers} represents the ideal condition for our theoretical analysis: considering the edge case of a large data set, we can filter out any source of epistemic uncertainty and, thus, study the effect of aleatoric uncertainty in isolation.

\myparagraph{Uncertainty convergence of standard training}
In a binary classification problem, the standard approach consists of using the cross-entropy loss $\ell(\x, y) = H(\vct{t}, f^{\param}(\x))$, corresponding to the entropy between the target $\vct{t}$ (representing the one-hot encoding of $y$) and the output of the classifier $f^{\param}(\x)$ parameterized with $\param$ evaluated on $\x$.
The empirical loss corresponds to the sum of the losses for each data point $\mathcal{L}(\set{D}) = \frac{1}{|\set{D}|} \sum_{i=1}^{|\set{D}|} \ell(\x_i, y_i)$.
Therefore, to minimize $\mathcal{L}(\set{D})$, one have to find a suitable parametrization $\theta^* = \arg \min_{\theta} \mathcal L(\mathcal D)$ minimizing the loss.

\begin{lemma}
\label{lemma:clean}
    A classifier $f^{\param}$ with parameters $\param$ minimizes the loss $\mathcal{L}(\set{D})$ when, for all $(\x,y) \in \set{D}$, it holds $f^{\param}(\x)=p(y|\x)$
\end{lemma}
\begin{proof}
Let $\partition{\set{D}} = \{\set{D}_{\x_0}, \dots \set{D}_{\x_{K}}\}$
be a partition of $\set{D}$ where each element $\set{D}_{\x_i} \coloneq \{(\x,y) \in \set{D} \suchthat \x=\x_i\}$ represents the set of all the samples having feature vector $\x_i$.
Since minimizing the loss for each $\set{D}_{\x} \in \partition{\set{D}}$ results in minimizing $\mathcal{L}(\set{D})$, we discuss the minimization of a single generic subset $\set{D}_{\x}$.
Such subset contains either elements belonging to class $A$ or $B$: accordingly, let the one hot encoded targets be $\vct{t_A}\coloneq[1,0]^\intercal$ and $\vct{t_B}\coloneq[0,1]^\intercal$.
Let then $\vct z=[z, 1-z]^\intercal$ be the vector associated to the categorical distribution of $p(y|\x)$, \ie a vector where $z$ and $1-z$ equal respectively $p(A|\x)$ and $p(B|\x)$.
If a classifier $f^{\param}(x)$ evaluated on $x$ outputs a probability vector $\vct{\alpha}\coloneq[\alpha, 1-\alpha]^\intercal$, one can expand the loss term on the subset $\set{D}_x$ as follows:
\begin{align*}
    \mathcal{L}(\set{D}_{\x}) &=
    \underbrace{p(\classA|\x)}_z \cdot \underbrace{H(\vct{t_A}, f^{\param}(\x))}_{- \log (\alpha)} + \underbrace{p(\classB|\x)}_{1-z} \cdot \underbrace{H(\vct{t_B}, f^{\param}(\x))}_{- \log (1-\alpha)} \\
    &= \quad -z \cdot \log (\alpha) - (1-z) \log (1-\alpha) = H(\vct z, \vct \alpha).
\end{align*}
Finally, since $H(\vct z, \vct \alpha)$ is minimized when $\vct z = \vct \alpha$, one can conclude that the loss $\mathcal{L}(\set{D})$ on the entire data set is minimized when this condition holds for all $(\x,y) \in \set{D}$ in the data set.
\end{proof}

\autoref{lemma:clean} represents a well-known result of probabilistic machine learning~\cite{Blasiok2023}, stating that the optimal classifier should be perfectly calibrated.
This phenomenon should be considered when dealing with aleatoric uncertainty because it asserts that lower entropy does not always go along with more accurate classifiers.
From this starting point, we will now show that such conditions slightly change in the presence of an attacker.

\myparagraph{Uncertainty convergence of adversarial training}   
Unlike standard training, the data we use during the minimization problem of adversarial training changes when the weight configuration changes: this happens because the perturbations applied to the input samples used during training depend directly upon the weights $\param^{(t)}$ at each iteration $t$.
From this perspective, one can extend the notation used with clean data to construct a dataset containing adversarial examples for training the model.
Let $\set{D}$ be the clean data set, and let $\advset{D}$ be the adversarial data set built upon $\set{D}$:
\begin{equation}
    \advset{D}^{(t)} = \{(\x+\delta^{(t)}, y) \suchthat (\x,y) \sim \datagendistrib\},
    \label{eq:advset}
\end{equation}
where $\delta^{(t)}$ is an adversarial perturbation applied to $\x$, obtained by optimizing \autoref{eq:advtraining} with the parameters $\param^{(t)}$ at iteration $t$.
Since we are considering a binary classification problem, one has to distinguish two cases for the same $\x$: if $y=A$, then \autoref{eq:advtraining} is minimized when $f^{\param}(x+\delta^{(t)})$ equals $\vct{t_B}\coloneq[0,1]^\intercal$, whereas if $y=B$, it is minimized when $f^{\param}(x+\delta^{(t)})$ equals $\vct{t_A} \coloneq [1,0]^\intercal$.

The minimization-maximization process leads to an equilibrium where the classifier would not improve its robustness by updating its parameters $\param^{(t)}$, because even though any update may eventually fix some vulnerabilities, it would at the same time expose the model to new ones.
Therefore, for each $(\x,y)$, this equilibrium theoretically bounds the loss an attacker can achieve when minimizing $p(y=c|\x)$ for a given target $c \in \set{Y}$.
Specifically, if we let the attack strength for the target class $c$ be $\beta_c = \frac{{\|f^{\param(\x)}_c - f^{\param(\x+\delta)}_c\|}_2}{\sqrt{2}}$, the equilibrium would be described by the following lemma:
\begin{lemma}
An adversarial trained classifier $f^{\robustparam}$ with robust parameters $\robustparam$, minimizes the loss $\mathcal{L}^{\robustparam}(\advset{D})$ when, for all $(\x+\delta,y) \in \advset{D}$ it holds $\alpha = z -\beta (2  z - 1)$, where $\beta = \max (\beta_A, \beta_B)$.
\label{lemma:adv}
\end{lemma}
\begin{proof}
Let $\partition{\advset{D}} = \{\advset{D}^{\x_0}, \dots ,\advset{D}^{\x_{K}}\}$ be a partition of $\advset{D}$ where $\advset{D}^{\x} \coloneq \{(\x+\delta,y) \in \advset{D} \suchthat (\x,y) \in \set{D}_{\x}\}$.
One can compute the loss on this partition as follows:
\begin{align*}
    \mathcal{L}^{\robustparam}(\advset{D}^{\x}) 
    &= \underbrace{p(\classA|\x+\delta_{\classA})}_z \cdot \underbrace{H(\vct{t_A}, f^{\param}(\x+\delta_{\classA}))}_{- \log (\alpha-\beta_A)} + \\
    &+ \underbrace{p(\classB|\x+\delta_{\classB})}_{1-z} \cdot \underbrace{H(\vct{t_B}, f^{\param}(\x+\delta_{\classB}))}_{- \log (1-\alpha-\beta_B)} \\
    &= - z \cdot \log (\alpha-\beta_A) - (1-z) \cdot \log (1-\alpha-\beta_B).
\end{align*}
Note that, for simplicity, one can consider the worst case by taking as attack strength $\beta = \max (\beta_A, \beta_B)$; in this configuration, we can write the loss as follows:
\begin{equation}
    \mathcal{L}^{\robustparam}(\advset{D}^{\x}) = - z \cdot \log (\alpha-\beta) - (1-z) \cdot \log (1-\alpha-\beta).
\end{equation}
Contrary to \autoref{lemma:clean}, when using adversarial training, the loss does not equal a cross-entropy between two distinct components.
Therefore, we need to compute the values for which the partial derivative with respect to $\alpha$ equals zero:
\begin{align*}
    \implies \frac{\partial \mathcal{L}}{\partial \alpha} = 0
    \implies \frac{1-z}{1-\alpha-\beta} - \frac{z}{\alpha-\beta} = 0\\
    \implies \alpha = z - \beta \cdot (2z-1)
\end{align*}
\end{proof}

\subsection{Demystifying the Underconfidence of Robust Models}
From \autoref{lemma:clean} and \autoref{lemma:adv}, it is possible to derive that, in the presence of an attacker, a classifier inevitably deviates from the trajectory of the optimization process that it would have followed in a standard, clean optimization process.
Interestingly, the entropy of an adversarial-trained classifier is higher compared to standard training.
We can easily prove this inequality starting from the two lemmas:
\begin{theorem}
    \label{theorem:adv_entropy_inequality}
    Let $f^{\param_S}$ and $f^{\param_R}$ be two classifiers trained respectively with standard and adversarial training on the data set $\set{D}$.
    In the presence of an attacker with attack strength $\beta \geq 0$, for all $(x,y) \in \set{D}$ holds that:
    \begin{equation}
        H(f^{\param_S}(x)) \leq H(f^{\param_R}(x))
    \end{equation}
\end{theorem}
\begin{proof}
From \autoref{lemma:clean}, we know $f^{\param_S}(x) = [z,1-z]^\intercal$, while from \autoref{lemma:adv} we know that in the presence of an attacker, the optimum deviates from an offset amount of $o=\beta  (2  z - 1)$: $f^{\param_R}(x) = [z-o,1-z+o]^\intercal$.
Within this interval, the entropy $H(f^{\param_R}(x))$ reaches its minimum when the offset equals zero and, thus, when $f^{\param_S}(x)=f^{\param_R}(x)$; therefore, $H(f^{\param_S}(x)) \leq H(f^{\param_R}(x))$.
\end{proof}

From \autoref{theorem:adv_entropy_inequality}, we can draw some conclusions on the robustness against uncertainty attacks by studying the effect of adversarial training on the Uncertainty Span.
Albeit ~\autoref{theorem:adv_entropy_inequality} proves the tendency of adversarial training to produce underconfident models, it remains factual the robustness of such models against \textit{traditional} adversarial examples, \ie crafted to induce a misclassification on a sample~\cite{madry2017towards}.
In order to succeed, the adversarial example requires crossing the decision boundary; in a binary classification problem, such conditions correspond precisely to reaching the maximum possible entropy.
Consequently, reducing the number of samples for which a traditional adversarial example exists (\ie increasing the robust accuracy) directly implies reducing the upper bound of the uncertainty span.
Concerning the robustness against over-confidence attacks, findings from ~\autoref{theorem:adv_entropy_inequality} suggest greater robustness against such attacks as well.
This theorem shows that adversarial-trained models output predictions with higher entropies than undefended models.
Suppose we aim to decrease it until reaching a desired value of $v$: in that case, the initially higher entropy of the adversarial trained model makes it more challenging to lower it to $v$, as the starting values are already higher and, therefore, need a significant reduction.

\myparagraph{Discussion}
In safe environments (without any attacker), a model whose purpose is to solve a classification problem takes advantage of correctly modeling the true posterior $p(y|\x)$ for any given input sample.
Nevertheless, in the presence of an attacker, the classifier tends to balance good generalization capabilities and security guarantees.
Such a balance compromises not only the model's accuracy (as we already know from previous literature) but also comes at the price of sacrificing a certain amount of confidence in the predictions.
As a result, adversarial-trained models disclose robust characteristics against uncertainty attacks, providing theoretical support for the claim that adversarial training makes models more resilient to over- and under-confidence attacks.
%
%
\section{Experimental Analysis}
\label{sec:experiments}
We now present the experimental setup employed to empirically validate our theoretical findings by analysing the Uncertainty Span on a selection of \NumRobustBenchModels state-of-the-art models from the well-known Robustbench repository~\cite{croce2021robustbench}.
In addition, we also consider a concrete case study where an attacker is interested in tampering with the model's calibration, proving the robustness of adversarial trained models for a real scenario. 
The code for the reproducibility of the experiments is available at the following link: \url{https://github.com/pralab/UncertaintyAdversarialRobustness} 

\subsection{Experimental Setup}
\label{sec:experiments.setup}

\myparagraph{Datasets}
We consider two well-known datasets used for classification tasks, \cifar~\cite{krizhevsky2009learning}, and \imagenet~\cite{krizhevsky2012imagenet}. 
We generate 8,000 adversarial examples from each dataset by randomly sampling from their respective test sets.

\myparagraph{Models}
For a comprehensive investigation, we consider 15 \cifar (denoted as \engstromID-\pangrobustnessID) and 8 \imagenet (denoted as \engstromimagenetID-\liuconvnl) standard robust classifiers from the Robustbench repository \robustbench, encompassing many different model architectures and robust accuracies, 
which are summarized in~\autoref{tab:architectures}.
We consider a wide range of different network architectures, including Convolutional Neural Networks (CNN) and Visual Tranformers (ViT); specifically, for analysing the uncertainty robustness when the model's size changes, we selected a wide range of model capacity, including ResNet-18, ResNet-50, ResNet-152, WideResNet-28-10, WideResNet-70-10, RaWideResNet-28-10, Swin-B, Swin-L, ConvNeXt-B and ConvNeXt-L.
Many of them rely on data augmentation (\addepallieffID ~\addepallieff, \gowalimprovingID, \gowalimprovingIDdue ~\gowalimproving, \wangbetterID,\wangbetterIDdue ~\gowalimproving, \rebuffifixingID~\rebuffifixing, \liuswinb, \liuswinl, \liuconvnb, \liuconvnl~\liu) or transfer learning with synthetic data (\sehwagrobustID, \sehwagrobustIDdue ~\sehwagrobust, \salmanID, \salmanIDdue~\salman ), while only two classifiers perform adversarial training using solely the original set (\engstromID, \engstromimagenetID~\engstrom).
Some work proposes more sophisticated strategies such as: (i) weight averaging (\rebuffifixingID~\rebuffifixing, \liuswinb, \liuswinl, \liuconvnb, \liuconvnl~\liu); (ii) ODE networks (\kangstableID~\kangstable); (iii) robust residual blocks (\pengrobustID~\pengrobust); 
(iv) modifications on the optimization process (\addepallitowID~\addepallitow, \cuidecoupledID~\cuidecoupled, \xuexploringID~\xuexploring, \pangrobustnessID~\pangrobustness).

\begin{table}[!h]
    \centering
    \resizebox{0.9\linewidth}{!}{%
    \begin{tabular}{lll}
    \toprule
    \textbf{ID} & \textbf{RobustBench Name} & \textbf{Architecture} \\
    \midrule
    \midrule
    \multicolumn{3}{c}{\textbf{CIFAR-10}} \\ 
    \midrule
    \engstromID \engstrom & \engstromname & \engstromarch \\
    \rowcolor{lightgray}
    \addepallieffID \addepallieff & \addepallieffname & \addepallieffarch \\
    \gowalimprovingID \gowalimproving & \gowalimprovingname & \gowalimprovingarch \\
    \rowcolor{lightgray}
    \gowalimprovingIDdue \gowalimproving & \gowalimprovingduename & \gowalimprovingduearch \\
    \wangbetterID \wangbetter & \wangbettername & \wangbetterarch \\
    \rowcolor{lightgray}
    \wangbetterIDdue \wangbetter & \wangbetterduename & \wangbetterduearch \\
    \sehwagrobustID \sehwagrobust & \sehwagrobustname & \sehwagrobustarch \\
    \rowcolor{lightgray}
    \sehwagrobustIDdue \sehwagrobust & \sehwagrobustduename & \sehwagrobustduearch \\
    \rebuffifixingID \rebuffifixing & \rebuffifixingname & \rebuffifixingarch \\
    \rowcolor{lightgray}
    \kangstableID \kangstable & \kangstablename & \kangstablearch \\
    \pengrobustID \pengrobust & \pengrobustname & \pengrobustarch \\
    \rowcolor{lightgray}
    \addepallitowID \addepallitow & \addepallitowname & \addepallitowarch \\
    \cuidecoupledID \cuidecoupled & \cuidecoupledname & \cuidecoupledarch \\
    \rowcolor{lightgray}
    \xuexploringID \xuexploring & \xuexploringname & \xuexploringarch \\
    \pangrobustnessID \pangrobustness & \pangrobustnessname & \pangrobustnessarch \\
    \midrule
    \multicolumn{3}{c}{\textbf{ImageNet}} \\ 
    \midrule
    \rowcolor{lightgray}
    \engstromimagenetID \engstrom & \engstromimagenetname & \engstromimagenetarch \\
    \salmanID \salman & \salmanname & \salmanarch \\
    \rowcolor{lightgray}
    \salmanIDdue \salman & \salmanduename & \salmanduearch \\
    \wongfastID \wongfast & \wongfastname & \wongfastarch \\
    \rowcolor{lightgray}
    \liuswinb \liu & \liuswinbname & \liuswinbarch \\
    \liuswinl \liu & \liuswinlname & \liuswinlarch \\
    \rowcolor{lightgray}
    \liuconvnb \liu & \liuconvnbname & \liuconvnbarch \\
    \liuconvnl \liu & \liuconvnlname & \liuconvnlarch \\
    \bottomrule
    \end{tabular}%
    }
    \caption{The RobustBench models used for our analysis. For more details on the specific hyperparameters used during the adversarial training phase, we refer the reader to the original papers.}
    \label{tab:architectures}
\end{table}

\myparagraph{Adversarial setup}
We implement the under-confidence and over-confidence attacks by employing the cross-entropy loss between the perturbed input and the targets specified in ~\autoref{sec:method}.
We then leverage the Projected Gradient Descent (PGD) method \cite{madry2017towards} for crafting the adversarial perturbation.
When running the attack, we constrain the \linf norm of the adversarial perturbation with a budget $\epsilon$, which we set as $8/255$ for \cifar and $4/255$ for \imagenet, accordingly to their original RobustBench evaluation.
Lastly, we set the number of attack iterations to \atkit, which is slightly higher compared to standard approaches used in literature for traditional evasion attacks targeting the predictions~\cite{croce2021robustbench, croce2020reliable}. Moreover, we empirically observed that this is suitable as all example losses reached convergence.

\myparagraph{OOD and Open-Set setup}
\label{sec:eval-ood}
To demonstrate how the analysis of the uncertainty span can be applied to concrete uncertainty attack scenarios, we instantiate our framework in both out-of-distribution (OOD) and Open-Set recognition (OSR) settings.
As discussed in \autoref{sec:method.approaches} and first proposed in \cite{zeng2022outdomain}, these attacks are designed to disguise samples that are semantically outside the training distribution, making them appear indistinguishable from in-distribution data.
To this end, we conduct two types of experiments:
\begin{itemize}
\item OOD scenario: we use 800 samples from the \mnist~\cite{deng2012mnist} dataset, which are clearly outside the training distribution;
\item OSR scenario: we select 800 samples from 10 novel classes in the \cifarh dataset. These classes are visually similar to the original 10 \cifar classes but are not part of the training set, thus simulating an open-set scenario~\cite{scheirer2012openset}.
\end{itemize}
For each scenario, we construct two datasets composed of 8800 samples (8000 in-distribution and 800 outliers): (i) the unperturbed dataset, and (ii) the dataset in which the outliers are perturbed with an over-confidence attack.
Finally, we compute the entropy of all the samples in the dataset, which is used to discern in-distribution samples from the outliers.

\myparagraph{Evaluation metrics}
\label{sec:eval-metrics}
In our empirical evaluation, we use the \mus and \msus measures discussed in ~\autoref{sec:method}.
For the calibration case study, however, a standard calibration measure such as the ECE would capture the miscalibration error but not tell if the model is under- or over-confident.
Since our study deals with increasing and decreasing the confidence scores of the adversarial samples, it is crucial to use a measure that captures such a phenomenon.
Accordingly, we propose a \textit{signed} version of the Expected Calibration Error, which we refer to as \sece (\autoref{eq:sece}), by taking the signed subtraction between the expected and the observed probabilities (thus, without the absolute value used for the standard ECE).
\begin{equation}
\text{\sece} = \sum^S_{s=1} \frac{|\bucket{s}|}{|\set{D}|} ({ex}(\bucket{s})-{ob}(\bucket{s})) \,
\label{eq:sece}
\end{equation}
By doing so, negative \sece values correspond to over-confidence in the models' predictions, whereas positive values correspond to under-confidence.

For the OOD and OSR experiments, we follow common evaluation practices~\cite{yang2021generalized, yang2022openood} by measuring the AUROC, AUPR (IN/OUT), and FPR95TPR.

\subsection{Experimental Results}
\label{sec:experiments.results}
This section presents the experimental results of our evaluation, which we first show through our entropy distribution plots and reliability diagrams and then analyze and discuss via the proposed metrics. 

\myparagraph{Entropy distribution shift}
To evaluate the entropy of the models before and after the attacks, we show here the entropy distribution plots, which reveal substantial differences in how the different robust models react to the adversarial attacks.
Through~\autoref{fig:violin-cifar}, we first depict the entropy distributions of the different models before and after the attacks, highlighting, interestingly, a similar behavior across models.
\begin{figure}[!tb]
    \centering
    \includegraphics[width=1\linewidth]{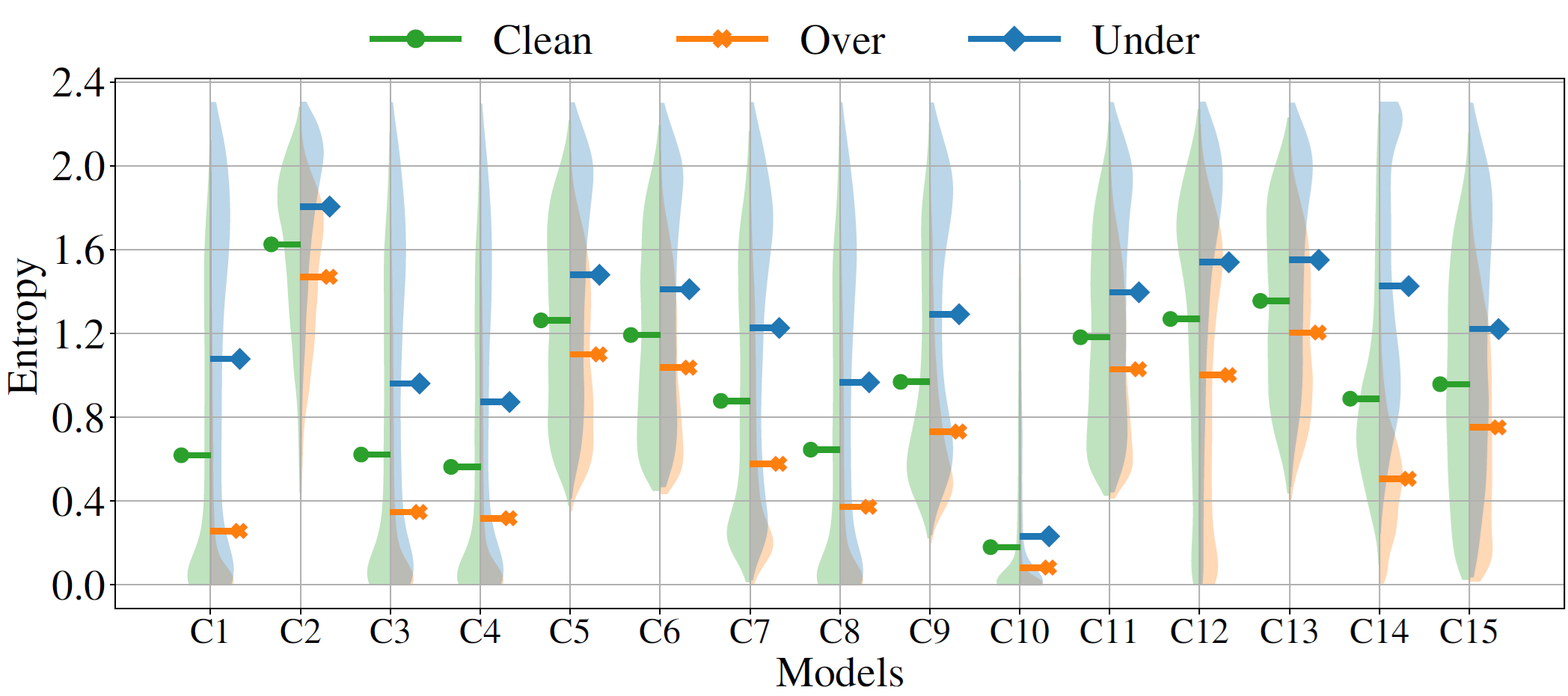}
    \caption{Entropy distribution before (green on the left) and after the over- (orange on the right) and under-confidence (blue on the right) attacks for all 15 \cifar robust classifiers, along with a line marking the mean of each distribution.}
    \label{fig:violin-cifar}
\end{figure}
In \engstromID, \gowalimprovingID, \gowalimprovingIDdue, and \sehwagrobustIDdue (which we denote with group \cifargroupA), the mass of the entropy distribution before the attack is primarily concentrated in low-entropy values, with a mode near $\sim 0.1$ and a mean of $\sim 0.6$.
The over-confidence attack primarily influences the samples having higher entropies, moving the mass toward the low-entropy mode.
The distribution after the under-confidence attack, instead, reveals a bimodal trend: apparently, some of the samples keep lower entropy values, while some others drastically move toward higher entropy regions; such a phenomenon generates two modes on the new distribution, one around $\sim 0.1$ and one near $\sim 1.6$.
\wangbetterID, \wangbetterIDdue, \pengrobustID, and \cuidecoupledID (group \cifargroupB) are characterized by a near-uniform distribution before the attack confined between $0.6$ and $1.8$.
After the attacks, the distributions gain a soft positive (for the over-confidence attack) and negative (for the under-confidence attack) skewness, shifting the mean of about $0.4$ from its clean value.
Even though we cannot group the remaining methods, some present the behaviors reported above, such as the bimodal distribution for the under-confidence distribution (\sehwagrobustID, \rebuffifixingID, \xuexploringID).
%
\addepallieffID and \kangstableID are the best-performing models: in \addepallieffID, the distribution softly shifts when under attack, while \kangstableID remains almost unchanged.
Interestingly, while \kangstableID presents very low clean entropy, \addepallieffID has one of the highest clean entropies among the models: such behavior reflects their low and high clean accuracies, respectively, reported on the RobustBench leaderboard~\cite{croce2021robustbench}.

Just like \cifar, in \imagenet, we report similarities across models, which allows us to divide them into groups.
\begin{figure}[!tb]
    \centering
    \includegraphics[width=1\linewidth]{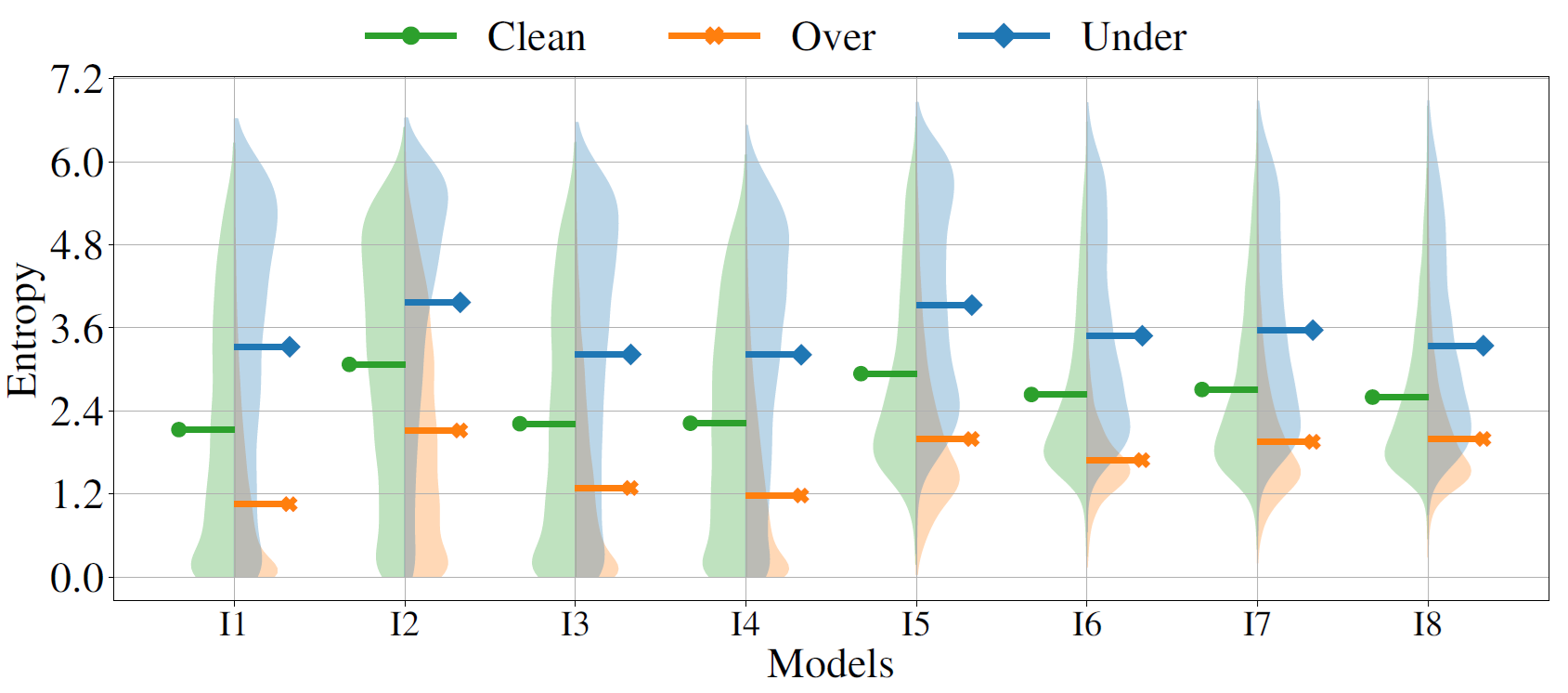}
    \caption{Entropy distribution before (green on the left) and after the over- (orange on the right) and under-confidence (blue on the right) attacks for all 8 \imagenet robust classifiers, along with a line marking the mean of each distribution.}
    \label{fig:violin-imagenet}
\end{figure}
The clean entropy distributions of \engstromimagenetID, \salmanIDdue, and \wongfastID (group \imagenetgroupA) present a mode of around $\sim 0.2$; however, a significant part of their mass is also distributed into a considerable interval of medium entropy levels, mainly between $1.0$ and $5.0$.
The over-confidence attacks successfully shift a significant part of the mass onto the mode around $\sim 0.2$.
After the under-confidence attack, similarly, a substantial mass shift, generating a mode around $\sim 5.2$.
The transformer architectures (\liuswinb, \liuswinl, \liuconvnb, and \liuconvnl) (group \imagenetgroupB) seem less sensitive to both attacks, particularly to over-confidence attacks.
Indeed, only \liuswinb presents a mode around higher entropy values, so we expect only \liuswinl, \liuconvnb, and \liuconvnl to be generally more robust than the remaining classifiers.

\myparagraph{Impact on the model's calibration}
We now show the effect of the over- and under-confidence attacks on the calibration of adversarial trained models.
\autoref{fig:calibration-cifar} and \autoref{fig:calibration-imagenet} show the reliability diagrams of the \NumRobustBenchModels robust models for both the analysed data sets, while \autoref{tab:sece_cifar} and \autoref{tab:sece_imagenet} report their \sece scores.
\begin{figure}[!tbp]
    \centering
    \includegraphics[width=\linewidth]{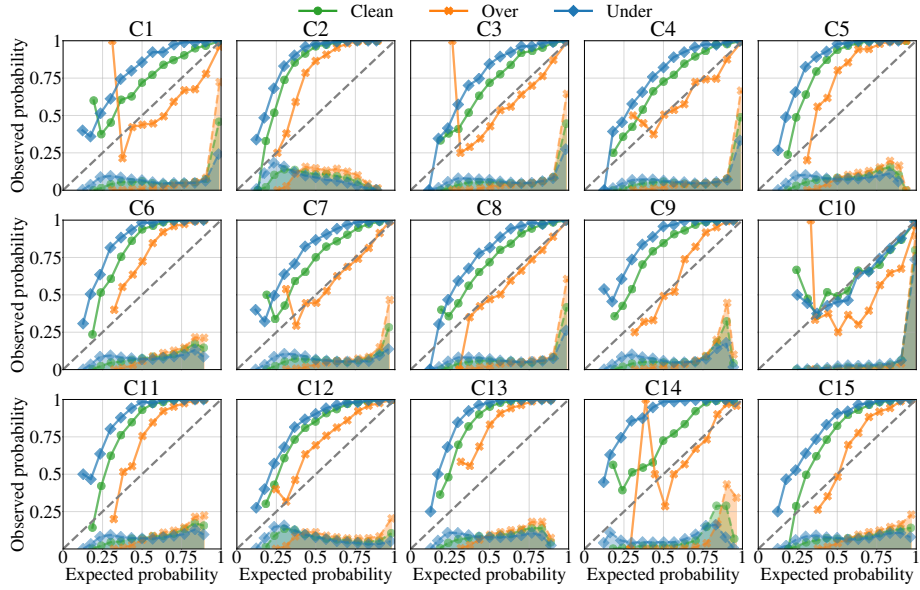}
    \caption{Calibration curves (above) and confidence histograms (below) for each \cifar model before (green) and after the over- (orange) and under-confidence (blue) attacks.}
    \label{fig:calibration-cifar}
\end{figure}
\begin{figure}[htbp]
    \centering
    \includegraphics[width=0.85\linewidth]{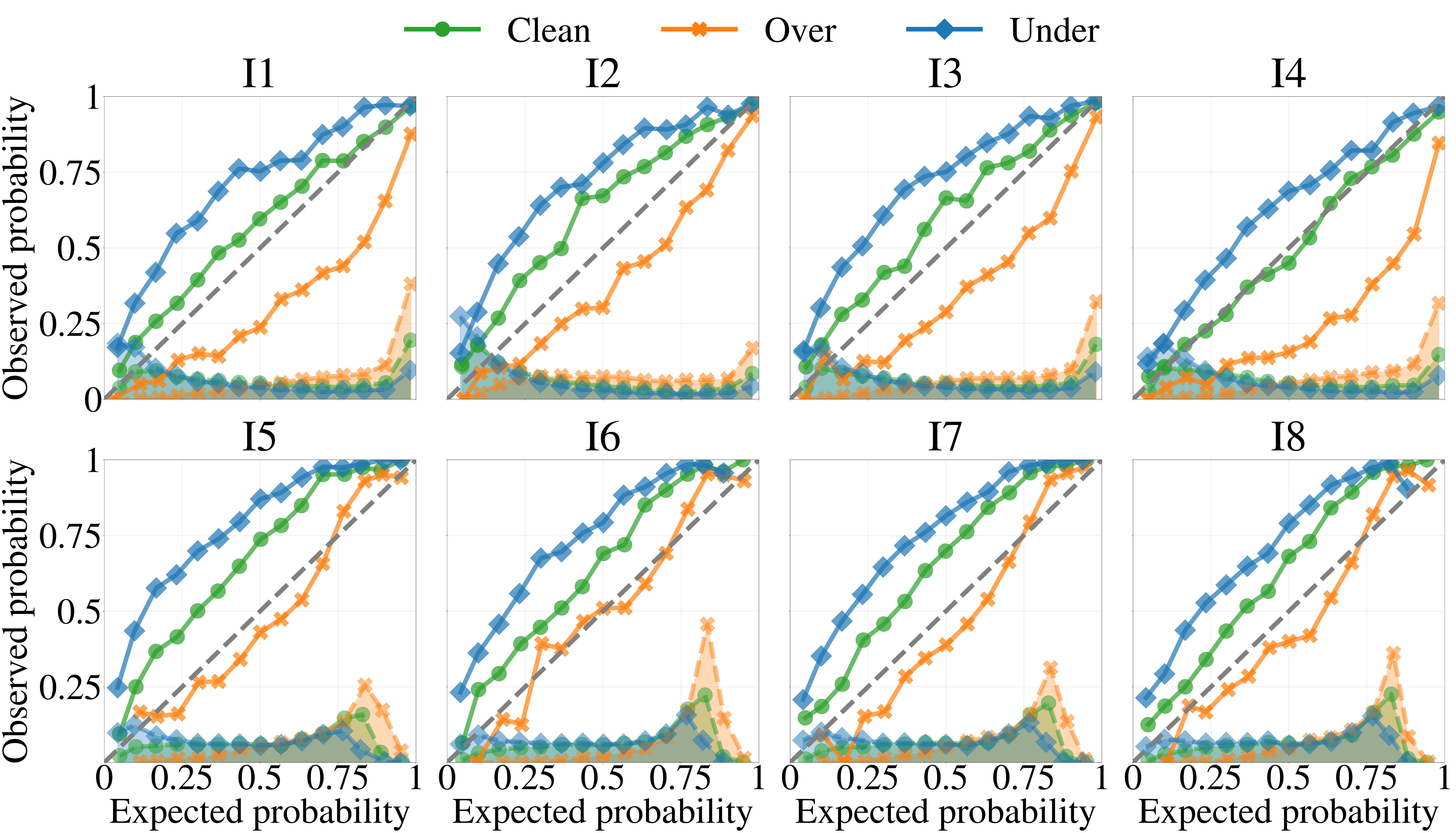}
    \caption{Calibration curves (above) and confidence histograms (below) for each \imagenet model before (green) and after the over- (orange) and under-confidence (blue) attacks.}
    \label{fig:calibration-imagenet}
\end{figure}
\begin{table}[!tbh]
\centering
\caption{\sece before and after the over- and under-confidence attacks for each \cifar model.}
\resizebox{\linewidth}{!}{
\setlength{\dashlinedash}{5pt}
\setlength{\dashlinegap}{3pt}
\begin{tabular}{lcccccccccccccccc}
\multicolumn{1}{c}{\textbf{\sece}}   & \multicolumn{1}{c}{\engstromID} & \multicolumn{1}{c}{\addepallieffID} & \multicolumn{1}{c}{\gowalimprovingID} & \multicolumn{1}{c}{\gowalimprovingIDdue} & \multicolumn{1}{c}{\wangbetterID} & \multicolumn{1}{c}{\wangbetterIDdue} & \multicolumn{1}{c}{\sehwagrobustID} & \multicolumn{1}{c}{\sehwagrobustIDdue} & \multicolumn{1}{c}{\rebuffifixingID} & \multicolumn{1}{c}{\kangstableID} & \multicolumn{1}{c}{\pengrobustID} & \multicolumn{1}{c}{\addepallitowID} & \multicolumn{1}{c}{\cuidecoupledID} & \multicolumn{1}{c}{\xuexploringID} & \multicolumn{1}{c}{\pangrobustnessID}\\ 
\toprule
 Under   & 0.22 & 0.44 & 0.19 & 0.17 & 0.37 & 0.39 & 0.25 & 0.19 &  0.30 & -0.01 & 0.34 & 0.33 & 0.39 & 0.33 & 0.26 \\
\cline{2-17}
 Clean   & 0.09 & 0.39 & 0.10 & 0.09 & 0.29 & 0.27 & 0.14 & 0.11 & 0.19 & -0.01 & 0.27 & 0.26 & 0.32 & 0.15 & 0.20\\
\cline{2-17}
 Over   & -0.06 & 0.29 & -0.02 & -0.01 & 0.21 & 0.20 & 0.00 & -0.01 & 0.09 & -0.05 & 0.19 & 0.13 & 0.24 & 0.04 & 0.10\\
\bottomrule
\end{tabular}}
\label{tab:sece_cifar}
\end{table}
\begin{table}[!tbh]
\centering
\caption{\sece before and after the over- and under-confidence attacks for each \imagenet model.}
\resizebox{0.65\linewidth}{!}{
\setlength{\dashlinegap}{3pt}
\begin{tabular}{lccccccccc}
\multicolumn{1}{c}{\textbf{\sece}} & \multicolumn{1}{c}{\engstromimagenetID} & \multicolumn{1}{c}{\salmanID} & \multicolumn{1}{c}{\salmanIDdue} & \multicolumn{1}{c}{\wongfastID} & \multicolumn{1}{c}{\liuswinb} & \multicolumn{1}{c}{\liuswinl} & \multicolumn{1}{c}{\liuconvnb} & \multicolumn{1}{c}{\liuconvnl} \\ 
\toprule
 Under   & 0.20 & 0.20 & 0.20 & 0.11 & 0.31 & 0.27 & 0.27 & 0.24 \\
\cline{2-10}
 Clean   & 0.06 & 0.11 & 0.08 & -0.01 & 0.19 & 0.16 & 0.17 & 0.16 \\
\cline{2-10}
 Over   & -0.20 & -0.12 & -0.15 & -0.28 & 0.01 & 0.07 & 0.01 & 0.02 \\
\bottomrule
\end{tabular}}
\label{tab:sece_imagenet}
\end{table}
The majority of the clean reliability curves are positioned above the identity line\footnote{Note that the spikes in models \engstromID, \gowalimprovingID, \kangstableID, and \xuexploringID do not significantly affect the \sece, due to the negligible impact of their small mass on the weighted sum.}, which empirically supports the findings of \autoref{theorem:adv_entropy_inequality}, \ie adversarial trained models tend to be under-confident.
The ``under-confident nature'' of adversarial trained models makes them naturally robust against over-confidence attacks (in line with our conjecture in ~\autoref{sec:theory}): indeed, as we can see from \autoref{tab:sece_cifar} and \autoref{tab:sece_imagenet}, even after the attack only a few \cifar models and half of the ImageNet models reach negative \sece values (\ie become over-confident).

\myparagraph{Discussion}
In \autoref{tab:leaderboard} we present a leaderboard ranking models based on their robustness against uncertainty attacks, sorted by \mus.
For each of the reported values we also show the standard deviation computed on the \mus and \msus distributions, for strengthen the statistical significance of the results.
\begin{table}[!htb]
\centering
\caption{Mean ($\mu$) and standard deviation ($\sigma$) of the \textbf{\mus} and \textbf{\msus} of each \cifar and \imagenet model (sorted by $\textbf{\mus}^\dag$), along with their accuracies and RobustBench-based rank.
Denoted with * we report potentially unreliable values, as discussed in \autoref{sec:experiments}.
} 
\resizebox{0.99\linewidth}{!}{ 
\begin{tabular}{clrccccccc} 
\label{tab:leaderboard} 
&  &  & \multicolumn{2}{c}{\textbf{Accuracy}} & \multicolumn{2}{c}{$\textbf{\mus}^\dag$} & \multicolumn{2}{c}{\textbf{\msus}} \\
\cmidrule(lr){4-5} \cmidrule(lr){6-7} \cmidrule(lr){8-9}
& \textbf{Robustbench ID}& \textbf{Model ID} & \textbf{Clean} & \textbf{Robust} & $\textbf{\vct  \mu}^\dag$ & \textbf{$\vct \sigma$} & \textbf{$\vct  \mu$} & \textbf{$\vct  \sigma$} & \textbf{RB rank}   \\ \hline

    \multirow{15}{*}{\rotatebox{90}{\large \cifar}}& Kang2021Stable  & \kangstableID \kangstable        & $93.73\%$ & $64.20\%$ & $0.149^*$ & $0.267^*$ & $0.094^*$ & $0.260^*$ & $7$  \\  
    \rowcolor{lightgray}
    \cellcolor{white} & Addepalli2022Efficient\_RN18 & \addepallieffID \addepallieff                                      & $85.71\%$ & $52.48\%$ & $0.335$ & $0.120$ & $0.127$ & $0.096$ & $13$  \\ 
    & Cui2023Decoupled\_WRN-28-10 \cifargroupB & \cuidecoupledID \cuidecoupled                          & $92.16\%$ & $67.73\%$ & $0.347$ & $0.168$ & $0.148$ & $0.152$ & $3$   \\
    \rowcolor{lightgray}
    \cellcolor{white} & Peng2023Robust \cifargroupB & \pengrobustID \pengrobust                                           & $93.27\%$ & $71.07\%$ & $0.368$ & $0.208$ & $0.178$ & $0.200$ & $1$   \\ 
    & Wang2023Better\_WRN-70-16 \cifargroupB & \wangbetterIDdue \wangbetter                             & $93.25\%$ & $70.69\%$ & $0.374$ & $0.208$ & $0.183$ & $0.201$ & $2$   \\
    & \cellcolor{lightgray}Wang2023Better\_WRN-28-10 \cifargroupB & \cellcolor{lightgray}\wangbetterID \wangbetter                                & \cellcolor{lightgray}$92.44\%$ & \cellcolor{lightgray}$67.31\%$ & \cellcolor{lightgray}$0.380$ & \cellcolor{lightgray}$0.192$ & \cellcolor{lightgray}$0.181$ & \cellcolor{lightgray}$0.188$ & \cellcolor{lightgray}$4$   \\ 
    & Pang2022Robustness\_WRN70\_16 \cifargroupAvar & \pangrobustnessID \pangrobustness                 & $89.01\%$ & $63.35\%$ & $0.470$ & $0.221$ & $0.270$ & $0.248$ & $10$  \\
    & \cellcolor{lightgray}Addepalli2021Towards\_RN18 & \cellcolor{lightgray}\addepallitowID \addepallitow                                        & \cellcolor{lightgray}$80.24\%$ & \cellcolor{lightgray}$51.06\%$ & \cellcolor{lightgray}$0.539$ & \cellcolor{lightgray}$0.293$ & \cellcolor{lightgray}$0.376$ & \cellcolor{lightgray}$0.398$ & \cellcolor{lightgray}$14$  \\ 
    & Gowal2021Improving\_70\_16\_ddpm\_100m \cifargroupA & \gowalimprovingIDdue \gowalimproving        & $88.74\%$ & $66.10\%$ & $0.555$ & $0.468$ & $0.527$ & $0.706$ & $6$   \\
    \rowcolor{lightgray}
    \cellcolor{white} & Rebuffi2021Fixing\_70\_16\_cutmix\_extra & \rebuffifixingID \rebuffifixing                        & $92.23\%$ & $66.56\%$ & $0.560$ & $0.339$ & $0.428$ & $0.441$ & $5$   \\ 
    & Sehwag2021Proxy\_ResNest152 \cifargroupA& \sehwagrobustIDdue \sehwagrobust                        & $87.30\%$ & $62.79\%$ & $0.594$ & $0.443$ & $0.550$ & $0.660$ & $11$  \\
    \rowcolor{lightgray}
    \cellcolor{white} & Gowal2021Improving\_28\_10\_ddpm\_100m \cifargroupA & \gowalimprovingID \gowalimproving           & $87.50\%$ & $63.38\%$ & $0.614$ & $0.470$ & $0.597$ & $0.738$ & $9$   \\ 
    & Sehwag2021Proxy\_R18 & \sehwagrobustID \sehwagrobust                                              & $84.59\%$ & $55.54\%$ & $0.649$ & $0.394$ & $0.576$ & $0.601$ & $12$  \\
    \rowcolor{lightgray}
    \cellcolor{white} & Engstrom2019Robustness \cifargroupA & \engstromID \engstrom                                       & $87.03\%$ & $49.25\%$ & $0.821$ & $0.579$ & $1.009$ & $1.063$ & $15$  \\ 
    & Xu2023Exploring\_WRN-28-10 & \xuexploringID \xuexploring                                          & $93.69\%$ & $63.89\%$ & $0.920$ & $0.551$ & $1.149$ & $1.216$ & $8$   \\
\midrule
    \rowcolor{lightgray}
    \cellcolor{white}\multirow{8}{*}{\rotatebox{90}{\large \imagenet}}&
    
      Liu2023Comprehensive\_ConvNeXt-L \imagenetgroupB & \liuconvnl \liu                                & $78.02\%$ & $58.48\%$  & $1.348$ & $0.936$ & $2.692$ & $3.558$ & $2$  \\
    & Liu2023Comprehensive\_ConvNeXt-B \imagenetgroupB & \liuconvnb \liu                                & $76.02\%$ & $55.82\%$  & $1.613$ & $1.020$ & $3.640$ & $4.311$ & $4$  \\ 
    & \cellcolor{lightgray}Liu2023Comprehensive\_Swin-L \imagenetgroupB & \cellcolor{lightgray}\liuswinl \liu                                     & \cellcolor{lightgray} \cellcolor{lightgray}$78.92\%$ & \cellcolor{lightgray}$59.56\%$  & \cellcolor{lightgray}$1.797$ & \cellcolor{lightgray}$1.304$ & \cellcolor{lightgray}$4.930$ & \cellcolor{lightgray}$6.075$ & \cellcolor{lightgray}$1$  \\
    & Salman2020Do\_R18 & \salmanID \salman                                                             & $52.92\%$ & $25.32\%$  & $1.850$ & $1.070$ & $4.568$ & $4.910$ & $8$  \\ 
    & \cellcolor{lightgray}Salman2020Do\_R50 \imagenetgroupA & \cellcolor{lightgray}\salmanIDdue \salman                                          & \cellcolor{lightgray}$64.02\%$ & \cellcolor{lightgray}$34.96\%$  &\cellcolor{lightgray} $1.928$ &\cellcolor{lightgray} $1.274$ & \cellcolor{lightgray}$5.341$ & \cellcolor{lightgray}$5.998$ & \cellcolor{lightgray}$5$  \\
    & Liu2023Comprehensive\_Swin-B \imagenetgroupBvar & \liuswinb \liu                                  & $76.16\%$ & $56.16\%$  & $1.934$ & $1.062$ & $4.867$ & $5.064$ & $3$  \\ 
    \rowcolor{lightgray}
    \cellcolor{white}& Wong2020Fast \imagenetgroupA & \wongfastID \wongfast                                              & $55.62\%$ & $26.24\%$  & $2.033$ & $1.243$ & $5.679$ & $5.862$ & $7$  \\
    & Engstrom2019Robustness \imagenetgroupA & \engstromimagenetID \engstrom                                      & $62.56\%$ & $29.22\%$  & $2.272$ & $1.477$ & $7.346$ & $7.700$ & $6$  \\ 

\bottomrule
\end{tabular} } 

\end{table} 

Here, the results reflect well our analysis: the group \cifargroupA comes after group \cifargroupB; the same applies for ImageNet, where almost all models from group \imagenetgroupA come after, except I5, the transformers (group \imagenetgroupB). 
Generally speaking - except for some outliers - we observe a slight correlation between robust accuracy and robust uncertainty: disregarding C10 and C2, group \cifargroupB has (on average) lower MUS and MSUS than group \cifargroupA, while the uncategorized models perform similarly to \cifargroupB; at the same time - without taking I5 into account - the Transformers outperform the ResNet architectures.
Nevertheless, the presence of \addepallieffID in the second place reveals that high entropic models may eventually disclose general robust properties against uncertainty attacks.

It is also important to note that \autoref{tab:leaderboard} reports the rankings obtained by RobustBench ordered by the best known robust accuracy.
Most of these robust accuracies are obtained by using AutoAttack, which is an ensemble of gradient-based white- and black-box attacks.
However, for certain models, even AutoAttack reveals to be unreliable: this is the case of \kangstableID, for which the best known robust accuracy was obtained by running a transfer attack.
In fact, methods that make use of Neural ODE in their architecture like \kangstableID are known to be affected by obfuscated gradients \cite{kang2021stable, Huang2021Adversarial}.
This phenomenon hinders the effectiveness of gradient-based attacks, leading to unreliable robustness evaluations~\cite{athalye2018obfuscated}.
Accordingly, since our evaluation consists of gradient-based attacks, we mark the results obtained with \kangstableID as unreliable, noting that our result may change by employing different attack strategies, \eg attack transferability, as suggested in \cite{kang2021stable, Huang2021Adversarial}.
%

\myparagraph{OOD and Open-Set results}
\autoref{tab:ood_openset} reports the results of the experiments for the OOD and OSR scenarios. 
\begin{table}[!htb]
    \centering
    \resizebox{0.99\linewidth}{!}{\begin{tabular}{cccccccccc}
    & & \multicolumn{4}{c}{\textbf{OSR}} & \multicolumn{4}{c}{\textbf{OOD}}  \\
    \cmidrule(lr){3-6} \cmidrule(lr){7-10}
    \textbf{Model ID} &      \textbf{$\epsilon$} &  \textbf{AUROC} &       \textbf{AUPR-IN}     &      \textbf{AUPR-OUT}     & \textbf{FPR95TPR} & \textbf{AUROC} &       \textbf{AUPR-IN}     &      \textbf{AUPR-OUT}     & \textbf{FPR95TPR}\\
    \midrule
    \rowcolor{white}
    \rebuffifixingID     & clean &      0.818 &             0.977 &         0.296 &         0.589 &          0.724 &                     0.963 &             0.189 &             0.710 \\
    \rowcolor{lightgray}
    \rebuffifixingID     & 4/255 &      0.741 &             0.966 &         0.175 &         0.695 &          0.681 &                     0.956 &             0.151 &             0.770 \\
    \rowcolor{white}
    \rebuffifixingID     & 8/255 &      0.642 &             0.951 &         0.116 &         0.782 &          0.635 &                     0.948 &             0.122 &             0.816 \\
    \midrule
    \rowcolor{lightgray}
    \wangbetterID        & clean &      0.823 &             0.977 &         0.318 &         0.596 &          0.668 &                     0.957 &             0.124 &             0.729 \\
    \rowcolor{white}
    \wangbetterID        & 4/255 &      0.750 &             0.967 &         0.191 &         0.687 &          0.635 &                     0.951 &             0.112 &             0.772 \\
    \rowcolor{lightgray}
    \wangbetterID        & 8/255 &      0.655 &             0.953 &         0.123 &         0.758 &          0.599 &                     0.944 &             0.102 &             0.805 \\
    \midrule
    \rowcolor{white}
    \cuidecoupledID      & clean &      0.816 &             0.976 &         0.298 &         0.602 &          0.602 &                     0.943 &             0.104 &             0.823 \\
    \rowcolor{lightgray}
    \cuidecoupledID      & 4/255 &      0.744 &             0.966 &         0.182 &         0.699 &          0.572 &                     0.937 &             0.096 &             0.858 \\
    \rowcolor{white}
    \cuidecoupledID      & 8/255 &      0.651 &             0.952 &         0.120 &         0.758 &          0.541 &                     0.930 &             0.090 &             0.885 \\
    \midrule
    \rowcolor{lightgray}
    \wangbetterIDdue     & clean &      0.836 &             0.979 &         0.334 &         0.571 &          0.651 &                     0.952 &             0.118 &             0.793 \\
    \rowcolor{white}
    \wangbetterIDdue     & 4/255 &      0.767 &             0.969 &         0.204 &         0.666 &          0.621 &                     0.946 &             0.109 &             0.828 \\
    \rowcolor{lightgray}
    \wangbetterIDdue     & 8/255 &      0.674 &             0.956 &         0.130 &         0.729 &          0.589 &                     0.940 &             0.101 &             0.858 \\
    \midrule
    \rowcolor{white}
    \gowalimprovingIDdue & clean &      0.830 &             0.980 &         0.296 &         0.513 &          0.777 &                     0.974 &             0.200 &             0.530 \\
    \rowcolor{lightgray}
    \gowalimprovingIDdue & 4/255 &      0.764 &             0.971 &         0.186 &         0.607 &          0.762 &                     0.972 &             0.184 &             0.549 \\
    \rowcolor{white}
    \gowalimprovingIDdue & 8/255 &      0.676 &             0.958 &         0.127 &         0.693 &          0.745 &                     0.969 &             0.170 &             0.570 \\
    \midrule
    \rowcolor{lightgray}
    \xuexploringID       & clean &      0.840 &             0.975 &         0.391 &         0.695 &          0.943 &                     0.994 &             0.614 &             0.206 \\
    \rowcolor{white}
    \xuexploringID       & 4/255 &      0.587 &             0.915 &         0.116 &         0.989 &          0.930 &                     0.993 &             0.542 &             0.244 \\
    \rowcolor{lightgray}
    \xuexploringID       & 8/255 &      0.348 &             0.851 &         0.065 &         1.000 &          0.914 &                     0.991 &             0.465 &             0.274 \\
    \bottomrule
    \end{tabular}}

    \caption{AUROC, AUPR and FPR95TPR computed on a mixture of ID-OSR (on the left) and ID-OOD (on the right) samples, for six Robust WideResNets, before and after the under-confidence attacks on the OOD and OSR subsets. }
    \label{tab:ood_openset}
\end{table}
%
As we can see from the results, all the considered metrics deteriorate in the presence of an attacker.
Nevertheless, even though such deterioration can vary on the base of the considered model, we notice that it does not lead to a complete break, as it happens for undefended models, where even for small perturbation budget it is possible to make the uncertainty measure on the OOD lower than any ID sample~\cite{Ledda_2023_ICCV}; this is a clear indicator of robustness for adversarially trained models against OOD and Open Set uncertainty attacks.
%
Quantitatively, we also notice that in almost all the cases, the adversarially trained models are able to discriminate better between open set samples compared to OOD samples.
The aleatoric nature of the considered uncertainty measure can be a possible explanation of this phenomenon.
Indeed, we remark that the entropy is an aleatoric measure, which is based upon the output prediction vector; notably, OOD represents an emblematic case of epistemic uncertainty, because, by definition, the considered model has no prior knowledge of the samples which are far from the training distribution. Instead, for the open set case, the lack of knowledge is less evident: even though the selected classes from \cifarh were not present in the training data, some visual features may possibly be shared across classes (\eg features from the ``helicopter'' class may be present in some \cifar veichle classes su as the ``airplane'' class).

Interestingly, we notice that \gowalimprovingIDdue and \xuexploringID are more sensitive to the open set attacks compared to the remaining ones.
In particular, \xuexploringID suffers from an extreme deterioration under attack in the open set case (\eg AUROC from $0.840$ ot $0.348$), which is in line with the ID leaderboard results in~\autoref{tab:leaderboard}.


%
\section{Conclusions, Limitations and Future Works}
\label{sec:conclusions}
In this work, we provided an extensive analysis of the robustness of adversarial-trained models against uncertainty attacks.
Our experimental evaluation conducted on image classification benchmarks has empirically validated our hypothesis that instead of needing an ad hoc defense strategy for uncertainty attacks, one can rely on already trained state-of-the-art robust models.
Nevertheless, we point out some limitations of our work, which we consider promising starting points for future research.
First, while our theoretical analysis reveals the interesting behaviour of adversarial trained models being underconfident by nature, this is still limited to 2-class cases and assumes the model finds a theoretical equilibrium.
Nevertheless, as happens often in Machine Learning~\cite{madry2017towards}, it is important to notice that all the discussed theoretical results should be taken not as an ultimate theory, but rather as a compass that can guide the research process toward conclusions which are interesting from an engineering perpective, even if the theoretical assumptions do not perfectly reflect the complexity of the real world.
Second, the attack setup that is shared across all considered models is not intended to be exhaustive from the adversarial perspective. For instance, we discussed in \autoref{sec:experiments.results} how \kangstable needs further investigation using more advanced techniques to properly assess its robustness~\cite{kang2021stable, Huang2021Adversarial, athalye2018obfuscated}.
To this end, our proposed leaderboard should not be used as the ultimate benchmark, but rather as a general compass that offers valuable insights on the uncertainty robustness capabilities of adversarial trained models.
Third, we prioritized our analysis to the sole aleatoric uncertainty, as previous research on adversarial robustness led to a plethora of pre-trained robust models that are deterministic by design, hence with no epistemic uncertainty available~\robustbench.
Accordingly, as future research, we intend to extend such analysis to epistemic uncertainty as well, particularly for proper Bayesian models, by employing seminal promising solutions that integrate adversarial training on top of Bayesian Neural Networks~\cite{carbone2021bayesianrobustness}.
On a closing note, we believe that our study represents a first step toward understanding the robustness of standard benchmark models under adversarial perturbations targeting uncertainty quantification and their impact on uncertainty estimates. The insights obtained open up promising directions for transferring these findings to emerging areas in AI, such as generative and foundation models, as well as to safety-critical domains like autonomous driving and healthcare. 


%
%
\section*{Acknowledgments}
\label{sec:ack}
This work has been supported by the European Union's Horizon Europe research and innovation program under the projects ELSA (grant no. 101070617) and CoEvolution (grant no. 101168560); by Fondazione di Sardegna under the project ``TrustML: Towards Machine Learning that Humans Can Trust’’, CUP: F73C22001320007; by EU-NGEU National Sustainable Mobility Center (CN00000023) Italian MUR Decree n. 1033—17/06/2022 (Spoke 10); and by projects SERICS (PE00000014) and FAIR (PE00000013) under the NRRP MUR program funded by the EU-NGEU.
This work was conducted while Emanuele Ledda, Daniele Angioni, and Giorgio Piras were enrolled in the Italian National Doctorate on Artificial Intelligence run by Sapienza University of Rome in collaboration with the University of Cagliari.





\bibliographystyle{elsarticle-num} 
\bibliography{main}





\end{document}